\documentclass{article}

\usepackage{microtype}
\usepackage{graphicx}
\usepackage{subfigure}
\usepackage{booktabs} 

\usepackage{hyperref}

\usepackage[accepted]{icml2020}

\usepackage[utf8]{inputenc} 
\usepackage[T1]{fontenc}    
\usepackage{hyperref}       
\usepackage{url}            
\usepackage{booktabs}       
\usepackage{amsfonts, amsmath, amssymb}       
\usepackage{nicefrac}       
\usepackage{microtype}      
\usepackage{xcolor}
\usepackage{comment}
\usepackage{amsthm}         
\usepackage{graphicx}
\usepackage{epstopdf} 
\usepackage{paralist, tabularx}

\newcommand{\reals}{\mathbf{R}}

\newcommand{\E}{\mathbf{E}}
\newcommand{\Tr}{\mathbf{Tr}}
\newcommand{\diag}{\mathbf{diag}}
\newcommand{\prox}{\mathbf{prox}}
\newcommand{\argmin}[1]{\underset{#1}{\text{argmin~}}}
\newcommand{\norm}[1]{\left\| #1\right\|}

\newtheorem{theorem}{Theorem}

\newtheorem{proposition}{Proposition}
\newtheorem{corollary}{Corollary}
\newtheorem{lemma}{Lemma}

\usepackage{algorithm}
\usepackage{algorithmic,bm,multirow,array}

\newcommand{\youngsuk}[1]{ {\color{red} {(\bf YS: #1)}} }
\newcommand{\rr}[1]{ {\color{blue} {(\bf RR: #1)}} }

\icmltitlerunning{Structured Policy Iteration for Linear Quadratic Regulator}

\begin{document}

%

%


\twocolumn[
\icmltitle{Structured Policy Iteration for Linear Quadratic Regulator}



\icmlsetsymbol{equal}{*}

\begin{icmlauthorlist}
\icmlauthor{Youngsuk Park}{stanford}
\icmlauthor{Ryan A.~Rossi}{adobe}
\icmlauthor{Zheng Wen}{google}
\icmlauthor{Gang Wu}{adobe}
\icmlauthor{Handong Zhao}{adobe}
\end{icmlauthorlist}

\icmlaffiliation{stanford}{Stanford University}
\icmlaffiliation{adobe}{Adobe Research} 
\icmlaffiliation{google}{Google DeepMind}

\icmlcorrespondingauthor{Youngsuk Park}{youngsuk@stanford.edu}

\icmlkeywords{Linear Quadratic Regulator, Low-rank structure, Regularized LQR, Reinforcement Learning}

\vskip 0.3in
]



\printAffiliationsAndNotice{}  

\begin{abstract}
Linear quadratic regulator (LQR) is one of the most popular frameworks to tackle continuous Markov decision process tasks. 
With its fundamental theory and tractable optimal policy, LQR has been revisited and analyzed in recent years, in terms of reinforcement learning scenarios such as the model-free or model-based setting. 
In this paper, we introduce the \textit{Structured Policy Iteration} (S-PI) for LQR, a method capable of deriving a structured linear policy. 
Such a structured policy with (block) sparsity or low-rank can have significant advantages over the standard LQR policy: more interpretable, memory-efficient, and well-suited for the distributed setting. 
In order to derive such a policy, we first cast a regularized LQR problem when the model is known. 
Then, our Structured Policy Iteration (S-PI) algorithm, which takes a policy evaluation step and a policy improvement step in an iterative manner, can solve this regularized LQR efficiently. 
We further extend the S-PI algorithm to the model-free setting where a smoothing procedure is adopted to estimate the gradient. 
In both the known-model and model-free setting, we prove convergence analysis under the proper choice of parameters. Finally, the experiments demonstrate the advantages of S-PI in terms of balancing the LQR performance and level of structure by varying the weight parameter.


\end{abstract}


\section{Introduction}
Stochastic control for the class of linear quadratic regulator (LQR) has been applied in a wide variety of fields including supply-chain optimization, advertising, dynamic resource allocation, and optimal control~\cite{
sarimveis2008dynamic,nerlove1962optimal,elmaghraby1993resource,anderson2007optimal} spanning several decades.

This stochastic control has led to a wide class of fundamental machinery along the way, across theoretical analysis as well as tractable algorithms, where the model of transition dynamic and cost function are known. On the other hand, under the uncertain model of transition dynamics, reinforcement learning (RL) and data-driven approaches have achieved a great empirical success in recent years, from simulated game scenarios  \cite{DBLP:journals/nature/MnihKSRVBGRFOPB15,DBLP:journals/nature/SilverHMGSDSAPL16} to robot manipulation \cite{tassa2012synthesis, al2012trajectory, kumar2016optimal}. In recent years, LQR in discrete time domain in particular, has been revisited and analyzed under model uncertainty, not only in theoretical perspective like regret bound or sample complexity \cite{ibrahimi2012efficient,fazel2018global,recht2019tour,mania2019certainty}, but also toward new real-world applications \cite{Lewis:2012,DBLP:conf/nips/LazicBLWRRI18,park2019linear}. 

Despite the importance and success of the standard LQR, discrete time LQR with a structured policy has been less explored in theoretical and practical perspective under both the known and unknown model settings, while such a policy may have a number of significant advantages over the standard LQR policy: interpretability, memory-efficiency, and is more suitable for the distributed setting.
%
In this work, we describe a methodology for learning a structured policy for LQR along with theoretical analysis of it. 

\paragraph{Summary of contributions.} 
\vspace{-3mm}
\begin{itemize}\itemsep0em
\vspace{-3mm}
    \item We formulate the regularized LQR problem for discrete time system that is able to capture a structured linear policy (in Section~\ref{sec:reg_lqr}).
    
    \item To solve the regularized LQR problem when the model is known, we develop the Structured Policy Iteration (S-PI) algorithm that consists of a policy evaluation and policy improvement step (in Section~\ref{sec:spi}).
    
    \item We further extend S-PI to the model-free setting, utilizing a gradient estimate from a smoothing procedure (in Section~\ref{sec:spi_model_free}). 
    

    \item We prove the linear convergence of S-PI to its stationary point under a proper choice of parameters in both the known-model (in Section~\ref{sec:spi_convergence}) and model-free settings (in Section~\ref{sec:spi_model_free_convergence}).
    
    \item We examine the properties of the S-PI algorithm and demonstrate its capability of balancing the LQR performance and level of structure by varying the weight parameter (in Section~\ref{sec:exp}). 
    
\end{itemize}

\subsection{Preliminary and background}\label{sec:prelim}
\paragraph{Notations.} 
For a symmetric matrix $Q\in \reals^{n \times n}$, we denote $Q \succeq 0$ and $Q \succ 0$ as a positive semidefinite and positive definite matrix respectively.
For a matrix $A\in \reals^{n \times n}$, we denote $\rho(A)$ as its spectral radius, i.e., the largest magnitude among eigenvalues of $A$. 
For a matrix $K\in \reals^{n \times m}$, $\sigma_1(K)\geq \sigma_2(K)\geq \cdots \geq \sigma_{\min(n,m)}(K)$ 
are defined as its ordered singular values where $\sigma_\mathrm{min}(K)$ is the smallest one. 
$\norm{A}$ denotes the $\ell_2$ matrix norm, i.e., $\max_{\norm{x}_2=1}\norm{Ax}_2$. We also denote Frobenius norm $\norm{A}_F=\sqrt{\sum_{i,j}A_{i,j}^2}$ induced by Frobenius inner product $\langle A, B\rangle_F = \Tr(A^TB)$ where $A$ and $B$ are matrices of the same size. A ball with the radius $r$ and its surface are denoted as $\mathbb{B}_r$ and $\mathbb{S}_r$ respectively.

\subsubsection{Optimal control for infinite time horizon.}
We formally define the class of problems we target here. 
Consider a Markov Decision Process (MDP) in continuous space where $x_t\in \mathcal{X} \subset \reals^n$, $u_t \in \mathcal{U} \subset \reals^m$, $w_t \in \mathcal{W} \subset \reals^n$ denote a state, an action, and some random disturbance at time $t$, respectively.  
Further, let $g:\mathcal{X} \times \mathcal{U} \times \mathcal{X} \rightarrow \reals$ be a stage-cost function, and $f:\mathcal{X} \times \mathcal{U} \times \mathcal{W} \rightarrow \mathcal{X}$ denote a transition dynamic when action $u_t$ is taken at state $x_t$ with disturbance $w_t$\footnote{often denoted as $w_{t+1}$ since it is fully revealed at time $t+1$.}. 

Our goal is to find the optimal stationary policy $\pi:\mathcal{X} \rightarrow \mathcal{U}$ that solves 
\begin{align}
    \underset{\pi}{\text{minimize} }  \quad & \sum_{t=0}^{\infty} \gamma^{t} \E [g(x_t, u_t, x_{t+1})] \label{eq:prbl_stat}\\
    \text{subject to} \quad&  x_{t+1} = f(x_t, u_t, w_t), ~x_0 \sim \mathcal{D} \nonumber
    \vspace{-5mm}
\end{align}

where $\gamma \in (0,1]$ is the discounted factor, $\mathcal{D}$ is the distribution over the initial state $x_0$. 


\subsubsection{Background on classic LQR}\label{sec:background_lqr}
Infinite horizon (discounted) LQR is a subclass of the problem \eqref{eq:prbl_stat} where the stage cost is quadratic and the dynamic is linear, $\mathcal{X}= \reals^n$, $\mathcal{U}=\reals^m$, $\gamma=1$, and $w_t=0$,  i.e., 
\begin{align}
    \underset{\pi}{\text{minimize} } \quad & \E\left(\sum_{t=0}^{\infty} x_t^T Q x_t + u_t^T R u_t \right) \label{eq:prbl_lqr}\\
    \text{subject to} \quad&  x_{t+1} = Ax_t + Bu_t, \nonumber\\
    \quad &   u_t = \pi(x_t), ~x_0 \sim \mathcal{D},\nonumber
\end{align}

where $A\in \reals^{n \times n}, B \in \reals^{n \times m}, Q \succeq 0$, and $ R \succ 0$.

\textbf{Optimal policy and value function.\;\;} The optimal policy (or control gain) and value function (optimal cost-to-go) are known to be linear and convex quadratic on state respectively,
\begin{align*}
    \pi^\star (x) = K x, \quad
    V^\star (x) = x^T P x  
\end{align*}
where
\begin{align}
    P &=  A^TP A + Q - A^T P B(B^TPB + R)^{-1} B^T P A, 
    \label{eq:dare}
    \\
    K &= - (B^TPB + R)^{-1} B^T P A.\nonumber
\end{align}
Note that Eq.~\eqref{eq:dare} is known as the discrete time Algebraic Riccati equation (DARE). Here, we assume $(A, B)$ is controllable\footnote{ The linear system is controllable if the matrix $[B, AB, \cdots, A^{n-1}B]$ has full column rank.}, then the solution is unique and can be efficiently computed via the Riccati recursion or some alternatives~\cite{hewer1971iterative,laub1979schur, lancaster1995algebraic,balakrishnan2003semidefinite}.

\textbf{Variants and extensions on LQR.\;\;} 
There are several variants of LQR including noisy, finite horizon, time-varying, and trajectory LQR. Including linear constraints, jumping and random model are regarded as extended LQR. In these cases, some pathologies such as infinite cost may occur \cite{bertsekas1995dynamic} but we do not focus on such cases.

\subsubsection{Zeroth order optimization.} 
Zeroth order optimization is the framework optimizing a function $f(x)$, only accessing its function values \cite{conn2009introduction,nesterov2017random}. It defines its perturbed function $f_{\sigma^2}(x)= \E_{\epsilon\sim \mathcal{N}(0, \sigma^2I)}f(x+\epsilon)$, which is close to the original function $f(x)$ for small enough perturbation $\sigma$. And Gaussian smoothing provides its gradient form $\nabla f_{\sigma^2}(x)= \frac{1}{\sigma^2}\E_{\epsilon\sim \mathcal{N}(0, \sigma^2I)}f(x+\epsilon)\epsilon$.  
Similarly, from \cite{flaxman2004online, fazel2018global}, we can do a smoothing procedure over a ball with the radius $r$ in $\reals^n$. The perturbed function and its gradient can be defined as $f_{r}(x)= \E_{\epsilon\sim \mathbb{B}_r}f(x+\epsilon)$ and $\nabla f_{r}(x)= \frac{n}{r^2} \E_{\epsilon\sim \mathbb{S}_r}f(x+\epsilon)\epsilon$ where $\epsilon$ is sampled uniformly random from a ball $\mathbb{B}_r$ and its surface $\mathbb{S}_r$ respectively. Note that these simple (expected) forms allow Monte Carlo simulations to get unbiased estimates of the gradient based on function values, without explicit computation of the gradient.

\subsection{Related work}
The general method for solving these problems is dynamic programming (DP) \cite{bellman1954theory,bertsekas1995dynamic}. 
To overcome the issues of DP such as intractability and computational cost, the common alternatives are approximated dynamic programming (ADP) \cite{bertsekas1996neuro, bertsekas2004stochastic,powell2007approximate,de2003linear, o2011min}
or Reinforcement Learning (RL)  \cite{sutton1998introduction} including policy gradient \cite{kakade2002natural, schulman2017proximal} or Q-learning based methods \cite{watkins1992q, DBLP:journals/nature/MnihKSRVBGRFOPB15}. 

\textbf{LQR.\;\;}
LQR is a classical problem in control that is able to capture problems in continuous state and action space, pioneered by Kalman in the late 1950's \cite{kalman1964linear}. Since then, many variations have been suggested and solved such as jump LQR, random LQR, (averaged) infinite horizon objective, etc \cite{florentin1961optimal,costa2006discrete, wonham1970random}. When the model is (assumed to be) known, Arithmetic Riccati Equation (ARE) \cite{hewer1971iterative,laub1979schur} can be used to efficiently compute the optimal value function and policy for generic LQR. 
Alternatively, we can use eigendecomposition~\cite{lancaster1995algebraic} or transform it into a semidefinite program (SDP)~\cite{balakrishnan2003semidefinite}.

\textbf{LQR under model uncertainty.\;\;}
When the model is unknown, one class is model-based approaches where the transition dynamics is attempted to be estimated. For LQR, in particular, \cite{abbasi2011regret,ibrahimi2012efficient} developed online algorithms with regret bounds where linear dynamic (and cost function) are learned and sampled from some confidence set, but without any computational demonstrations. Another line of work is to utilize robust control with system identification \cite{dean2017sample,recht2019tour} with sample complexity analysis \cite{tu2017least}. Certainty Equivalent Control for LQR is also analyzed with suboptimality gap \cite{mania2019certainty}. The other class is model-free approaches where policy is directly learned without estimating  dynamics. Regarding discrete time LQR, in particular, Q-learning~\cite{bradtke1994adaptive} and policy gradient~\cite{fazel2018global} were developed together with lots of mathematically machinery therein.  but with little empirical demonstrations.

\textbf{LQR with structured policy.\;\;}
For continuous time LQR, many work in control literature \cite{wang1973stabilization,sandell1978survey,lau1972decentralized} has been studied for sparse and decentralized feedback gain, supported by theoretical demonstrations. Recently, \cite{wytock2013fast, lin2012sparse} developed fast algorithms for a large-scale system that induce the sparse feedback gain utilizing Alternating Direction Method of Multiplier (ADMM), Iterative Shrinkage Thresholding Algorithm (ISTA), Newton method, etc. And, most of these only consider sparse feedback gain under continuous time LQR.

However, to the best of our knowledge, there are few algorithms for discrete time LQR, which take various structured (linear) policies such as sparse, low-rank, or proximity to some reference policy into account. Furthermore, regarding learning such a structured policy, there is no theoretical work that demonstrates a computational complexity, sample complexity, or the dependency of stable stepsize on algorithm parameters either, even in known-model setting (and model-free setting).

\section{Structured Policy for LQR}\label{sec:strc_lqr}

\subsection{Problem statement: regularized LQR}\label{sec:reg_lqr}

From standard LQR in Eq.~\eqref{eq:prbl_lqr}, we restrict policy to linear class, i.e., $u_t = K x_t$, and add a regularizer on the policy to induce the policy structure. We formally state a \textit{regularized LQR} problem as
\begin{align}
    \underset{K}{\text{minimize} }  \quad &  \overbrace{ \E\left(\sum_{t=0}^{\infty}
    x_t^TQx_t + u_t^TRu_t \right)}^{f(K)}  + \lambda r(K) \label{eq:regularized_lqr}\\
    \text{subject to} \quad&  x_{t+1} = Ax_t + Bu_t, \nonumber\\
    \quad &   u_t = K x_t, \quad x_0 \sim \mathcal{D}, \nonumber 
\end{align}
for a nonnegative parameter $\lambda\geq 0$. Here $f(K)$ is the (averaged) cost-to-go under policy $K$, and $r:\reals^{n \times m}\rightarrow \reals$ is a nonnegative convex regularizer inducing the structure of policy $K$. 

\paragraph{Regularizer.} Different regularizers induce different types of structures on the policy $K$. We consider lasso $r(K) = \|K\|_1=\sum_{i,j}|K_{i,j}|$, group lasso  $r(K) = \|K\|_{\mathcal{G}, 2} = \sum_{g\in \mathcal{G}}\norm{K_g}_2$ where $K_g \in \reals^{|g|}$ is the vector consisting of a index set $g$, and nuclear-norm $r(K) = \|K\|_* = \sum_{i} \sigma_i(K)$ where $\sigma_i(K)$ is the $i$th largest singular value of $K$. These induces sparse, block sparse, and low-rank structure respectively. For a given reference policy 
$K^\mathrm{ref}\in \reals^{n \times m}$, 
we can similarly consider $\|K- K^\mathrm{ref}\|_{1}$, $\|K- K^\mathrm{ref}\|_{\mathcal{G}, 2}$, and $\|K- K^\mathrm{ref}\|_{*}$, penalizing the proximity (in different metric)  to the reference policy $K^\mathrm{ref}$. 

\paragraph{Non-convexity.} For a standard (unregularized) LQR, the objective function $f(K)$ is known to be not convex, quasi-convex, nor star-convex, but to be gradient dominant. Therefore, all the stationary points are optimal as long as $\E[x_0x_0^T]\succ 0$ (see Lemma 2,3, in~\cite{fazel2018global}).
However, in regularized LQR, all the stationary points of Eq.~\eqref{eq:regularized_lqr} may not be optimal under the existence of multiple stationary points. (See the supplementary material for the detail.)

\subsection{Structured Policy Iteration (S-PI)}
\label{sec:spi}
Eq.~\eqref{eq:regularized_lqr} can be simplified into
\begin{align}
    \text{minimize} \quad & F(K) := f(K)+ \lambda r(K). \label{eq:regularized_lqr_simple}
\end{align}\noindent
Here $f(K) = \mathbf{Tr} (\Sigma_0 P)$ where $\Sigma_0 = \E [x_0 x_0^T]$ is the covariance matrix of initial state and $P$ is the quadratic value matrix satisfying the following Lyapunov equation
\begin{align}
(A+BK)^T P (A+BK) - P  + Q + K^TRK=0.\label{eq:lyapunov}
\end{align}\noindent
We introduce the \textit{Stuctured Policy Iteration} (S-PI) in Algorithm~\ref{alg:spi} to solve Eq.~\eqref{eq:regularized_lqr_simple}. 
The S-PI algorithm consists of two parts: 
(1) policy evaluation and (2) policy improvement.
In the policy evaluation part, we solve Lyapunov equations to compute the quadratic value matrix $P$ and covariance matrix $\Sigma$. 
In the policy improvement part, we try to improve the policy while encouraging some structure, via the proximal gradient method with proper choice of an initial stepsize and a backtracking linesearch strategy.

\begin{algorithm}
  \caption{\;{\bf{S}}tuctured {\bf{P}}olicy {\bf{I}}teration ({\bf{S-PI}})}  
    \label{alg:spi}
\begin{algorithmic}[1]
\STATE {\bfseries given} initial stable policy $K^0$ and initial state covariance matrix $\Sigma_0 = \E[x_0 x_0^T]$, linesearch factor $\beta < 1$.
\REPEAT 
  \STATE \textbf{(1) Policy (and covariance) evaluation:}
   \STATE compute $(P^i, \Sigma^i)$  satisfying Lyapunov equations 
   \begin{equation}
    \label{eq:spi_lyapunov}
    \begin{aligned}
            (A+BK^i)^T& P^i (A+BK^i) - P^i + Q + (K^i)^T R K^i 
            =0,
            \nonumber
            \\
            (A+BK^i)& \Sigma^i (A+B K^i)^T - \Sigma^i + \Sigma_0=0.
    \end{aligned}
   \end{equation}
    \textbf{return}  $(P^i, \Sigma^i)$ 
    
   \STATE \textbf{(2) Policy improvement:}
    \STATE initial stepsize $\eta_i = \mathcal{O}(1/\lambda)$.
    \STATE compute gradient at $K$
    \vspace{-2mm}
    \begin{align*}
        \nabla_K f(K^i) =  2\left(\left( R  + B^T P^i B\right)K^i +  B^T P^i A \right) \Sigma^i
    \end{align*}
    \vspace{-5mm}
    \REPEAT
       \STATE $\eta_i := \beta \eta_i$.
       \STATE $K^{i+1}{\small\leftarrow}\mathrm{ProxGrad}(\nabla f(K^i), \eta_i, r, \lambda)$ (in Alg.~\ref{alg:proximal_gradient}).
    \UNTIL{ linesearch \eqref{eq:linesearch} criterion is satisfied. }\\
    \textbf{return} next iterate $K^{i+1}.$ 
\UNTIL{stopping criterion $\|K^{i+1} - K^i\| \leq \epsilon$ is satisfied.}
\end{algorithmic}
\end{algorithm}


\begin{algorithm}
  \caption{Subroutine: $\mathrm{ProxGrad}(\nabla f(K), \eta, r, \lambda)$}
    \label{alg:proximal_gradient}
\begin{algorithmic}[1]
\STATE {\bfseries Input} gradient oracle $\nabla f(K)$, stepsize $\eta$, and regularization $r$ and its parameter $\lambda$

\STATE take gradient step 
\vspace{-2mm}
\begin{align*}
    G \leftarrow K - \eta \nabla_K f(K)
\end{align*}
\vspace{-6mm}
\STATE take proximal step
\vspace{-1mm}
\begin{align*}
    K^+ \leftarrow &\prox_{ r(\cdot),\lambda\eta}(G)\\
    &:= \argmin{K} r(K) + \frac{1}{2\lambda\eta} \|K-G\|_F^2
\end{align*}
\vspace{-5mm}
\STATE \textbf{return} $K^+$
\end{algorithmic}
\end{algorithm}

\textbf{Sensitivity to initial and fixed stepsize choice.}
Note that we use the initial stepsize $\eta = \mathcal{O}(1/\lambda)$ as a rule of thumb whereas the typical initial stepsize \cite{barzilai1988two, wright2009sparse, park2020variable} does not depend on the regularization parameter $\lambda$. 
This stepsize choice is motivated by theoretical analysis (see Lemma~\ref{lemma:stable_ball} in Section~\ref{sec:spi_convergence}) as well as empirical demonstration (see Fig.~\ref{fig:stepsize_dependency} in Section~\ref{sec:exp}). This order of stepsize automatically scales well when experimenting over various $\lambda$s, alleviating iteration counts and leading to a faster algorithm in practice. It turns out that proximal gradient step is very sensitive to stepsizes, often leading to an unstable policy $K$ with $\rho(A+BK)\geq1$ or requiring a large number of iteration counts to converge. Therefore, we utilize linesearch over fixed stepsize choice.

\textbf{Linesearch.\;\;}
We adopt a backtracking linesearch (See \cite{parikh2014proximal}). 
Given $\eta_i$, $K^i$, $\nabla f(K^i)$, and the potential next iterate $K^{i+1}$, it check if the following criterion (the stability and the decrease of the objective) is satisfied,
\begin{equation}
\label{eq:linesearch}
  \begin{aligned}
    f(K^{i+1}) 
    \leq 
    f(K^i) &
    - \eta_i \Tr (\nabla f(K^i)^T G_{\eta_i}(K^i))
    \\
    &+
    \frac{\eta_i}{2}\|G_{\eta_i}(K^i)\|_F^2,
    \\
    \rho(A + BK^{i+1})& < 1,
  \end{aligned}
\end{equation}
where $G_{\eta_i}(K) = \frac{1}{\eta_i}(K - \prox_{r,\lambda\eta_i}(K - \eta_i \nabla f(K) ) )$ and  $\rho(\cdot)$ is the spectral radius. Otherwise, it shrinks the stepsize $\eta_i$ by a factor of $\beta<1$ and check it iteratively until Eq.~\eqref{eq:linesearch} is satisfied. 

\textbf{Stabilizing sequence of policy.\;\;}
We can start with a stable policy $K^0$, meaning $\rho(A + B K^0)<1$.
For example, under the standard assumptions on $A, B, Q, R$, Riccati recursion provides a stable policy, the solution of standard LQR in Eq.~\eqref{eq:prbl_lqr}. Then, satisfying the linesearch criterion in Eq.~\eqref{eq:linesearch} subsequently, the rest of the policies $\{K^i\}$ are a stabilizing sequence. 

\textbf{Computational cost.\;\;}
The major cost incurs when solving the Lyapunov equations in the policy (and covariance) evaluation step. Note that if $A+BK$ is stable, so does $(A+BK)^T$ since they share the same eigenvalues. Under the stability, each Lyapunov equation in Eq.~\eqref{eq:spi_lyapunov} has a unique solution with the computational cost  $O(n^3)$~\cite{jaimoukha1994krylov, li2002low}. Additionally, we can solve a sequence of Lyapunov equations with less cost, via using iterative methods with adopting the previous one (warm-start) or approximated one. 


\subsubsection{Regularizer and proximal operator}
For various regularizers mentioned in Section~\ref{sec:reg_lqr}, each has the closed-form solution for its proximal operator \cite{rockafellar1976monotone,parikh2014proximal,park2020variable}. 
Here we only include a few representative examples and refer to the supplementary material for more examples. 

\begin{lemma}
    [Examples of proximal operator]
    \quad
   \begin{enumerate}
       \item 
        \textbf{Lasso.} For $r(K) = \norm{K}_1$,
        \begin{align*}
        	\left( \prox_{r,\lambda\eta}(K) \right)_{i,j} 
        	=
        	\mathrm{sign}(K_{i,j})(|K_{i,j}| - \lambda\eta)_+.
        \end{align*}
        And we denote $\prox_{r,\lambda\eta}(K) := S_{\lambda\eta}(K)$ as a soft-thresholding operator.
       \item 
        \textbf{Nuclear norm.} For $r(K) = \norm{K}_*$,
        \begin{align*}
        	 \prox_{r,\lambda\eta}(K)  =
        	 U \diag(S_{\lambda\eta}(\sigma)) V^T.
        \end{align*}
        where $K = U\diag(\sigma)V^T$ is singular value decomposition with singular values $\sigma \in \reals^{\min(n,m)}$.
    \item 
        \textbf{Proximity to $K^\mathrm{ref}$.} For $r(K) = \norm{K-K^\mathrm{ref}}_F^2$,
        \begin{align*}
        	\prox_{r,\lambda\eta}(K)  = 
        	\frac{2\lambda\eta K^\mathrm{ref} + K}{2\lambda\eta + 1}.
        \end{align*}
   \end{enumerate}
   
\end{lemma}

\subsection{Convergence analysis of S-PI}\label{sec:spi_convergence}

Recall that $\norm{\cdot}$, $\norm{\cdot}_F$, $\sigma_\mathrm{min}(\cdot)$ is $\ell_2$ matrix norm, Frobenius norm, and smallest singular value, as mentioned in Section~\ref{sec:prelim}.
First, we start with (local) smoothness and strong convexity around a stable policy. Here we regard a policy $K$ is stable if $\rho(A+BK)<1$.  

\textbf{Assumptions.\;} $\rho(A + BK^0)<1$, $\Sigma_0 =\E x_0x_0^T \succ 0$, $\norm{K^0 - K^\star}\leq \Delta$, and $\norm{K^{\mathrm{ref}} - K^\star}\leq \Delta$. 

\begin{lemma}\label{lemmma:smoothness}
For stable $K$, 
$f(K)$ is smooth (in local) with 
\begin{align*}
L_K = 4 \|\Sigma(K)\| \|R + B^T P(K) B\| < \infty,
\end{align*}
within local ball around $K \in \mathcal{B}(K;\rho_K) $
where the radius $\rho_K$ is 
\[
\rho_K = \frac{\sigma_{\mathrm{min}}(\Sigma_0)  } {4 \|\Sigma(K)\| \left(\|A + BK\| + 1\right) \|B\|} > 0.
\]
And $f(K)$ is strongly convex (in local) with 
\begin{align*}
m = \sigma_{\mathrm{min}}(\Sigma_0) \sigma_{\mathrm{min}}(R) > 0.
\end{align*}
\end{lemma}

Then, we provide a proper stepsize that guarantees one iteration of proximal gradient step is still inside of the stable and (local) smooth ball $\mathcal{B}(K;\rho_K)$. 

\begin{lemma}\label{lemma:stable_ball}
   Let $K^+ = \prox_{ r,\lambda\eta}(K - \eta \nabla f(K) )$. 
   Then 
   \[
    K^+ \in  \mathcal{B}(K ;\rho_K)
   \]
    holds for any $0 < \eta < \eta_K^{\lambda,r}$ where $\eta_K^{\lambda,r}$ is given as 
    \begin{align*}
    \eta_K^{\lambda,r}= 
    \begin{cases}
      \frac{\rho_K}{\|\nabla f(K)\| + \lambda nm}
      &~ r(K) = \|K\|_1
      \\
        \frac{\rho_K}{\|\nabla f(K)\| + \lambda \min(n,m)}
      &~  r(K) = \|K\|_*
      \\
        \frac{\rho_K}{2 \|\nabla f(K)\| + 2\lambda\norm{K - K^\mathrm{ref}}  }
      &~ r(K) = \|K- K^\mathrm{ref}\|_F^2
    \end{cases}.
    \end{align*}
       
       
\end{lemma}

Next proposition describes that next iterate policy has the decrease in function value and is stable under sufficiently small stepsize. 
\begin{proposition}
 
 \label{lemma:descent_lemma}
Assume $A + BK$ is stable. For any stepsize $0 < \eta \leq \min(\frac{1}{L_K}, \eta_K^{\lambda,r})$ and next iterate $K^+ = \prox_{r, \eta \lambda} (K - \eta \nabla f(K))$, 
\begin{align}
    \rho(A + BK^+) < 1 \label{eq:stable_next_iterate}
    \\
    F(K^+) \leq  F(K) - \frac{1}{2\eta}\|K - K^+\|_F^2\label{eq:descent}
\end{align}
holds.
\end{proposition}



We also derive the bound on stepsize in Lemma~\ref{lemma:stable_ball}, not dependent on iteration numbers. 
\begin{lemma}\label{lemma:bound_p_sigma}
Assume that $\{K^i\}_{i=0,\ldots}$ is a stabilizing sequence and associated $\{f(K^i)\}_{i=0,\ldots}$ and $\{\|K^i - K^\star \|_F\}_{i=0,\ldots}$ are decreasing sequences. 
Then, Lemma~\ref{lemma:stable_ball} holds for 
\begin{equation}
\label{eq:eta_r}
    \begin{aligned}
    \eta^{\lambda,r}= 
    \begin{cases}
      \frac{\rho^L}{\rho^f + \lambda nm}
      &~r(K) = \|K\|_1
      \\
        \frac{\rho^L}{\rho^f + \lambda \min(n,m)}
      &~  r(K) = \|K\|_*
      \\
        \frac{\rho^L}{2 \rho^f + 4\lambda\Delta  }
      &~ r(K) = \|K- K^\mathrm{ref}\|_F^2
    \end{cases}.
    \end{aligned}
\end{equation}
where 
\begin{align*}
\rho^f&=
2\frac{F(K^0)}{\sigma_{\min}(Q)} \bigg(
\norm{B^T}  \frac{F(K^0)}{\sigma_{\min}(\Sigma_0)} \norm{A} 
+\nonumber
\\  
&\left( \norm{R}  + \norm{B^T} \frac{F(K^0)}{\sigma(\Sigma_0)} \norm{B}\right)
(\Delta + \norm{K^\star})\bigg),
\\
\rho^L &= \frac{\sigma_{\mathrm{min}}(\Sigma_0)^2  } {8  F(K^0) \|B\| }.
\end{align*}

\end{lemma}

Now we prove that the S-PI in Algorithm~\ref{alg:spi} converges to the stationary point linearly.
\begin{theorem}\label{thm:linear_convergence}
$K^i$ from Algorithm~\ref{alg:spi} converges to the stationary point $K^\star$. Moreover, it converges linearly, i.e., after $N$ iterations,
\[
\norm{K^N - K^\star }_F^2 
\overset{}{\leq } \left(1 - \frac{1}{\kappa}\right)^N \norm{K^0 - K^\star}_F^2.
\]
Here, $\kappa=1/\left(\eta_{\min}\sigma_{\mathrm{min}}(\Sigma_0) \sigma_{\mathrm{min}}(R) ) \right)> 1$ where  
    \begin{align}
    \eta_{\min} =
    h_{\eta}\bigg(&\sigma_\mathrm{min}(\Sigma_0),
    \sigma_\mathrm{min}(Q),\frac{1}{\lambda},
    \nonumber
    \\&
    \frac{1}{\norm{A}},
    \frac{1}{\norm{B}},
    \frac{1}{\norm{R}},
    \frac{1}{\Delta},
    \frac{1}{F(K^0)}
    \bigg),
    \label{eq:stepsize}
    \end{align}
    for some  non-decreasing function $h_\eta$ on each argument.
\end{theorem}
Note that (the global bound on) stepsize $\eta_\mathrm{min}$ is inversely proportional to $\lambda$, which motivates the initial stepsize in linesearch.

\begin{corollary}\label{thm:linear_convergence_eps}
Let $K^\star$ be the stationary point from algorithm \ref{alg:spi}. 
Then, after $N$ iterations
\[
N \geq 2\kappa \log\left(\frac{\norm{K^0 - K^\star}_F}{\epsilon}\right),
\]
\[
\norm{K^N - K^\star }_F \leq \epsilon
\]
holds 
where $\kappa=1/\left(\eta_{\min}\sigma_{\mathrm{min}}(\Sigma_0) \sigma_{\mathrm{min}}(R) ) \right)> 1$ and  $\eta_\mathrm{min}$ in Eq.~\eqref{eq:stepsize}. 
\end{corollary}






\section{Toward a Model-Free Framework} \label{sec:strt_lqr_model_free}
In this section, we consider the scenario where the cost function and transition dynamic are unknown. Specifically in model-free setting, policy is directly learned from (trajectory) data without explicitly estimating the cost or transition model. 
In this section, 
we extend our Structured Policy Iteration (S-PI) to the model-free setting and prove its convergence.

\subsection{Model-free Structured Policy Iteration (S-PI)}
\label{sec:spi_model_free}
Note that, in model-free setting, model parameters $A, B, Q$ and $R$ cannot be directly accessed, which hinders the direct computation of $P$, $\Sigma$, and $\nabla f(K)$ accordingly. Instead, we adopt a smoothing procedure to estimate the gradient based on samples. 

Model-free S-PI in Algorithm~\ref{alg:spi_model_free} consists of two steps: (1) policy evaluation step and (2) policy improvement step. In (perturbed) policy evaluation step, perturbation $U^j$ is uniformly drawn from the surface of the ball with radius $r$, $\mathbb{S}_r \subset \reals^{n\times m}$. These data are used to estimate the gradient via a smoothing procedure for the policy improvement step. With this approximate gradient, proximal gradient subroutine tries to decrease the objective while inducing the structure of policy. Comparing to (known-model) S-PI in Algorithm~\ref{alg:spi}, one important difference is its usage of a fixed stepsize $\eta$, rather than an adaptive stepsize from a backtracking linesearch  that requires to access function value $f(K)=\Tr(\Sigma_0 P)$ explicitly.

\begin{algorithm}[h!]
  \caption{Model-free {\bf{S}}tuctured {\bf{P}}olicy {\bf{I}}teration (Model-free {\bf{S-PI}})}  
    \label{alg:spi_model_free}
\begin{algorithmic}[1]
\STATE {\bfseries given} initial stable policy $K^0$, number of trajectories $N_\mathrm{traj}$, roll-out horizon $H$, smoothing parameter $r$, and (fixed) stepsize $\eta$.
\smallskip
\REPEAT 
    \STATE \textbf{(1) (Perturbed) policy evaluation}:
    \FOR{$j=1,\ldots,N_{\mathrm{traj}}$}
      \STATE sample a perturbed policy $\hat K^{i}= K^i + U^j$ where  and $U^j\sim\mathrm{Uniform} (\mathbb{S}_r)$.
      \STATE roll out $\hat K^i$ from sampled initial state $x_0\sim \mathcal{D}$, over the horizon $H$ to estimate the cost-to-go
      \vspace{-2mm}
      \[
      \hat f^j  = \sum_{t=0}^H g_t
      \vspace{-2mm}
      \]
      where $g_t :=g(x_t, \hat K^i x_t)$ is the stage cost incurred at time $t$.
    \ENDFOR\\
    \textbf{return} cost-to-go and perturbation $\{\hat f^j, U^j\}_{j=1}^{N_{\mathrm{traj}}}$. 

\smallskip
   \STATE \textbf{(2) Policy improvement:}
    \STATE estimate  the gradient
    \vspace{-2mm}
    \begin{align}
        \widehat{\nabla_K f(K^i)} = 
        \frac{1}{N_\mathrm{traj}} \sum_{j=1}^{N_\mathrm{traj}} \frac{d}{r^2} \hat f^j U^j
        \label{eq:gradient_estimate}
    \end{align}
    \vspace{-5mm}
    \STATE $K^{i+1}$ $\leftarrow$  $\mathrm{ProxGrad}(\widehat{\nabla_K f(K^i)}, \eta, r, \lambda)$ (in Alg.~\ref{alg:proximal_gradient}).
          
    \textbf{return} next iterate $K^{i+1}.$ 
\UNTIL{stopping criterion $\|K^{i+1} - K^i\| \leq \epsilon$ is satisfied.}
\end{algorithmic}
\end{algorithm}

\subsection{Convergence analysis of model-free S-PI}
\label{sec:spi_model_free_convergence}
The outline of proof is as following: We first claim that for proper parameters (perturbation, horizon number, numbers of trajectory), the gradient estimate from the smoothing procedure is close enough to actual gradient with high probability. Next we demonstrate that approximate proximal gradient still converges linearly with high probability.  


\begin{theorem}\label{thm:linear_convergence_model_free}
Suppose $F(K^0)$ is finite, $\Sigma_0 \succ 0$, and that $x_0 \sim \mathcal{D}$ has norm bounded by $D$ almost surly. Suppose the parameters in Algorithm~\ref{alg:spi_model_free} are chosen from 
\[
(N_\mathrm{traj}, H, 1/r) = 
h\left(n,
\frac{1}{\epsilon}, 
\frac{1}{ \sigma_{\mathrm{min}}(\Sigma_0) \sigma_{\mathrm{min}}(R) }, \frac{D^2}{\sigma_{\mathrm{min}}(\Sigma_0)}\right).
\]
for some polynomials $h$.
Then, with the same stepsize in Eq.~\eqref{eq:stepsize},
there exist iteration $N$  at most $4\kappa \log\left(\frac{\norm{K^0 - K^\star}_F}{\epsilon}\right)$ such that $\norm{K^N - K^\star }\leq \epsilon$ with  at least $1 - {o}(\epsilon^{n-1})$ probability. Moreover, it converges linearly,
\[
\norm{K^i - K^\star }^2 
\overset{}{\leq } \left(1 - \frac{1}{2\kappa}\right)^i \norm{K^0 - K^\star}^2, 
\]
for the iteration $i=1,\ldots,N$,  where $\kappa=\eta\sigma_{\mathrm{min}}(\Sigma_0) \sigma_{\mathrm{min}}(R)  > 1$.
\end{theorem}
 
\paragraph{Remark.} For (unregularized) LQR, the model-free policy gradient method \cite{fazel2018global} is the first one that adopted a smoothing procedure to estimate the gradient.
Similar to this, our model-free S-PI for regularized LQR also has several major challenges for deploying in practice.
First one is that finding an initial policy with finite $F(K^0)$ is non-trivial, especially when the open loop system is unstable, i.e., $\rho(A) \geq  1$. Second one is its sensitivity to fixed stepsize $\eta$ as in the known model setting (in Section~\ref{sec:spi}), wrong choice of which easily makes it divergent.
Note that these two challenges hold over most of gradient based methods for LQR including policy gradient \cite{fazel2018global}, trust region policy optimization (TRPO)~\cite{schulman2015trust}, or proximal policy optimization (PPO)~\cite{schulman2017proximal}. 
Last one is the joint sensitivity to multiple parameters $N_\mathrm{traj}, H, r$, which may lead to high variance or large sub-optimality gap. 
On the other hand, REINFORCE \cite{williams1992simple} may suffer less variance, but with another potential difficulty: estimating the state-action value function $Q(x,u)$. Moreover, it is not clear how to derive a structured policy. 
However, here we adopted smoothing procedures that enable us to theoretically analyze the convergence rate and parameter dependency.

\section{Experiments} \label{sec:exp}
In experiments, we consider a LQR system for the purpose of validating the theoretical results and basic properties of the S-PI algorithm. As mentioned in Section~\ref{sec:spi_model_free_convergence}, the simple example with an unstable open loop system, i.e., $\rho(A)\geq1$, is extremely sensitive to parameters even under known model setting, 
which may make it less in favor of the generic model-free RL approaches to deploy. 
Please note that our objective is total LQR cost-to-go in Eq.~\eqref{eq:prbl_lqr} more difficult than LQR cost-to-go averaged over time-horizon that some of works \cite{mania2018simple} considered.  
Under this difficulty, we demonstrate the properties of S-PI such as parameter sensitivity, convergence behaviors, and capability of balancing between LQR performance and policy structures. 
Finally, we illustrate the scalability of algorithms over various system dimensions.

\subsection{Synthetic systems} \label{sec:exp-syn-sys}
In these experiments, we use the unstable Laplacian system ~\cite{recht2019tour}. 


\textbf{Large Laplacian dynamics.\;\;}
$A \in \reals^{n\times n}$ where 
\begin{align*}
    A_{ij}
    =
    \begin{cases}
        1.1, & i=j \\
        0.1, & i=j+1 \text{ or }j=i+1  \\
        0,   & \text{otherwise}
    \end{cases}
\end{align*}
$B = Q = I_n \in \reals^{n\times n}$ and $R = 1000 \times I_n \in \reals^{n\times n}$.


\textbf{Synthetic system parameters.\;\;}
For the Laplacian system, we regard $(n,m)=(3,3)$, $(n,m)=(20,20)$, and $(n,m)=(10^3,10^3)$ dimension as small, medium, and large size of system. 
In addition, we experiment with Lasso regularizer over various $\lambda=10^{-2}\sim 10^{6}$.


\subsection{Algorithm parameters}
Based on our theoretical results (Lemma~\ref{lemma:stable_ball}) and sensitivity experiments that empirically show the dependency of stepsize $\eta$ on $\lambda$ (in Fig.~\ref{fig:stepsize_dependency}), we set the initial stepsize $\eta = \frac{1}{\lambda}$.  
For the backtracking linesearch, we set $\beta = \frac{1}{2}$ and the convergence tolerance $\epsilon_{\mathrm{tol}}=10^{-6}$. We start with an initial policy $K^0$ from Riccati recursion. 

\subsection{Results}
We denote the stationary point from S-PI (at each $\lambda > 0$) as $K^{\star}$ and the solution of the (unregularized) LQR as $K^{lqr}$ .

\begin{figure}
\vspace{-0mm}
\centering
  \includegraphics[width=0.8\linewidth]{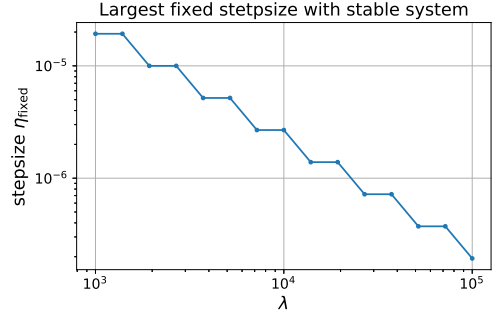}
  \vspace{-4mm}
  \caption{Largest fixed stepsize leading stable system as $\lambda$ varies. 
  This demonstrates that stepsize  $\eta_\mathrm{fixed}\propto \frac{1}{\lambda}$.
  }
  \label{fig:stepsize_dependency}
  \vspace{-2mm}
\end{figure}

\textbf{Dependency of stepsize $\eta$ on $\lambda$.\;\;}
Under the same problem but with different choices of weight $\lambda$, the largest fixed stepsize $\eta$ that demonstrates the sequence of stable systems, i.e., $A+BK^i<1$ actually varies, as Lemma~\ref{lemma:stable_ball} implies.
Fig.~\ref{fig:stepsize_dependency} shows that the largest stepsize diminishes as $\lambda$ increases. 
This motivates the choice of the initial stepsize (in linesearch) to be inversely proportional to $\lambda$.


\begin{figure}
\centering
\includegraphics[width=0.8\linewidth]{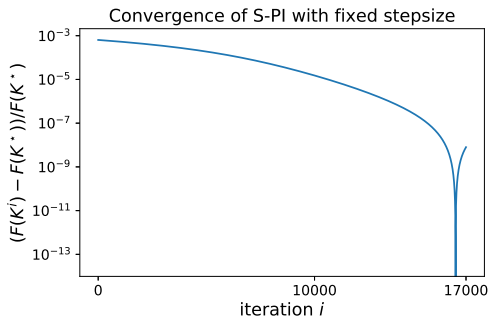}
\vspace{-5mm}
\caption{Convergence behavior of the Structured Policy Iteration (S-PI) with fixed stepsize for Laplacian system of $(n,m)=(3,3)$ with $\lambda=3000$. 
}
\label{fig:convergence}
\end{figure}

\textbf{Convergence behavior and stepsize sensitivity.\;\;}
S-PI with a linesearh strategy converges very fast, within 2-3 iterations for most of the Laplacian system with dimension $n=3\sim 10^3$ and weight $\lambda=10^{-2}\sim 10^6$. However, S-PI with fixed stepsize may perform with subtlety even though the fixed stepsize choice is common in typical optimization problems. 
In Fig.~\ref{fig:convergence}, S-PI with fixed stepsize for the small Laplacian system $(n,m)=(3,3)$ gets very close to optimal but begins to deviate after a certain amount of iterations.
Note that it was performed over long iterations with small enough stepsize $\eta_\mathrm{fixed}=3 \times  10^{-5}$. Please refer to supplementary materials for more detailed results over various stepsizes.
This illustrates that a linesearch strategy can be essential in S-PI (or even for the existing policy gradient method for LQR \cite{fazel2018global}), even though the convergence analysis was shown with some fixed stepsize (but difficult to compute in practice). 

\begin{figure}
\vspace{-3mm}
\includegraphics[width=1.0\linewidth]{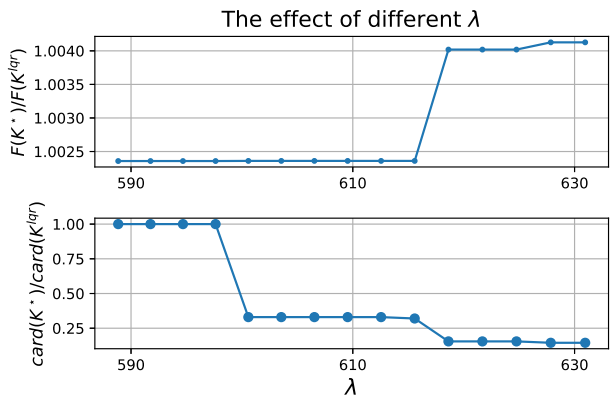}
\vspace{-6mm}
\caption{
As $\lambda$ become larger, the LQR cost slightly increases (top) within range $\lambda < 615$ whereas sparsity is significantly improved by $50\%$ (bottom) for a Laplacian system with $(n,m)=(20, 20)$.
}
\label{fig:result_laplacian}
\end{figure}

\textbf{Trade off between LQR performance and structure $K$.\;\;}
In Fig.~\ref{fig:result_laplacian} for a medium size Laplacian system, we show that as $\lambda$ increases, the LQR cost $f(K^\star)$ increases whereas cardinality decreases (sparsitiy is improved). 
Note that the LQR performance barely changes (or is slightly worse) for $\lambda \leq 615$ but the sparsity is significantly improved by more than $50\%$. 
In Fig.~\ref{fig:sparsity_laplacian}, we show the sparsity pattern (location of non-zero elements) of the policy matrix with $\lambda=600$ and $\lambda=620$.

\begin{figure}[t!]
\vspace{0mm}
\includegraphics[width=1.0\linewidth]{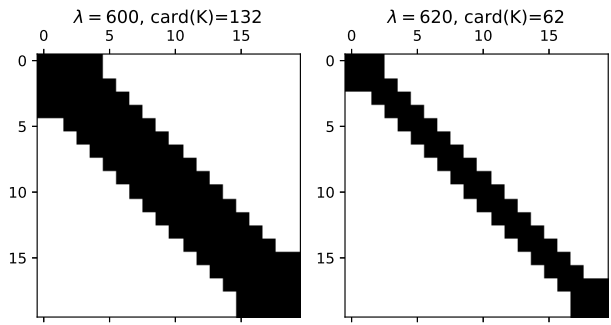}
\vspace{-6mm}
\caption{Sparse pattern of policy with $\lambda=600$ and $\lambda=620$ respectively for a Laplacian system with $n=20$.
}
\vspace{-2mm}
\label{fig:sparsity_laplacian}
\end{figure}

\textbf{Scalability \& runtime performance.\;\;}
In Fig.~\ref{fig:scalability}, we report the runtime for Laplacian system of $n=10,\ldots,500$.
Notably, it shows the scalability of S-PI, as it takes less than 2 minutes to solve a large system with $n=500$ dimensions, where we used a MacBook Air (with a 1.3 GHz Intel Core i5 CPU) for experiments. These results demonstrate its applicability to large-scale problems in practice.

\begin{figure}
\centering
\includegraphics[width=0.8\linewidth]{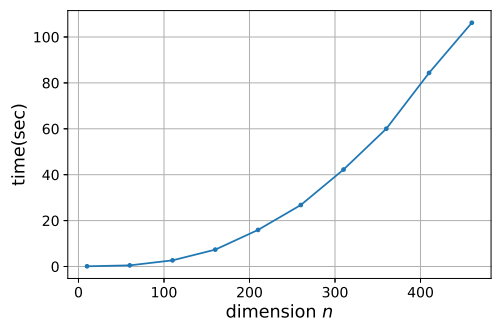}
\vspace{-5mm}
\caption{ 
The elapsed time (sec.) until S-PI converges over $n=10,\ldots, 500$ for the Laplacian system. 
}
\label{fig:scalability}
\vspace{-3mm}
\end{figure}




\section{Conclusion and Discussion}

In this paper, we formulated a regularized LQR problem to derive a structured linear policy and provided an efficient algorithm, Structured Policy Iteration (S-PI). 
We proved that S-PI guarantees to converge linearly under a proper choice of stepsize, keeping the iterate within the set of stable policy as well as decreasing the objective at each iteration. 
In addition, we extended S-PI for model-free setting, utilizing a smoothing procedure. 
We also proved its convergence guarantees with high probability under a proper choice of parameters including stepsize, horizon counts, trajectory counts, etc. In the experiments, we examined some basic properties of the S-PI such as sensitivity on stepsize and regularization parameters to convergence behaviors, which turned out to be more critical than typical optimization problems. Lastly, we demonstrated that our method is effective in terms of balancing the quality of solution and structures. 

We leave for future work the practical application of other penalty functions such as low-rank and proximity regularization. There are also new extensions on using proximal Newton or proximal natural gradient method as a subroutine, beyond what was developed in S-PI in Section~\ref{sec:strc_lqr}, which could be further analyzed. Even though model-free algorithm for regularized LQR was suggested with theoretical guarantees, it is extremely difficult to deploy in practice, like most of model-free approaches for LQR. Finally, developing algorithm reducing variance for (regularized) LQR, possibly like \cite{papini2018stochastic, park2020linear}, as well as some practical rule of thumb on the choice of hyper-parameters is another class of important problems to tackle toward model-free settings. 

We described a new class of discrete time LQR problems that have yet to be studied theoretically and practically. And we discussed and demonstrated how such problems can be of practical importance despite them not being well-studied in the literature.While a few application were covered for this new class of problems, each of these contributions would open up our framework to new potential applications, providing additional benefits to future research on this topic.


\newpage
\bibliographystyle{icml2020}
\bibliography{refs.bib}


\appendix
\section{Discussion on Non-convexity of Regularized LQR.}

From Lemma 2 and appendix in \cite{fazel2018global}, unregularized objective $f(K)$  is known to be not convex, quasi-convex, nor star-convex, but to be gradient dominant, which gives the claim that all the stationary points are optimal as long as $\E[x_0x_0^T]\succ 0$ .
However, in regularized LQR, this claim may not hold. 

To see this claim that all stationary points may not be global optimal, let's define regularized LQR with $r(K) = \norm{K-K^\mathrm{lqr}}$ where $K^\mathrm{lqr}$ is the solution of the Riccati algorithm. We know that  $K^\mathrm{lqr}$ is the global optimal. Assume there is another distinct stationary point (like unregularized LQR) $K^\prime$. Then,  $f(K^\mathrm{lqr})+\lambda r(K^\mathrm{lqr}) = f(K^\mathrm{lqr})$ is always less than $f(K^\prime) + \lambda \norm{K^\prime - K^\mathrm{lqr}}$. If not,i.e., $f(K^\mathrm{lqr}) \geq f(K^\prime) + \lambda \norm{K^\prime - K^\mathrm{lqr}}$, then $f(K^\prime) < f(K^\mathrm{lqr})$ holds and this is contradiction, showing all stationary points is not global optimal like unregularized LQR. Whether regularized LQR has only one stationary point or not is still an open question.

\section{Additional Examples of Proximal Operators}

Assume $\lambda, \lambda_1, \lambda_2 \in \reals_{+}$ are positive numbers. We denote $(z)_i \in \reals$ as its $i$th element or $(z)_j \in \reals^{n_j}$ as its $j$th block under an explicit block structure, and $(z)_+ = \max(z,0)$. 

\begin{itemize}
    \item \textbf{Group lasso.} For a group lasso penalty $r(x) = \sum_{j}^N\|x_j\|_2$ with $x_j \in \reals^{n_j}$,
    \begin{align*}
    	\left(\prox_{r,\lambda \eta}(x)\right)_j = \left(1 - \frac{\lambda\eta}{\|x_j\|_2}\right)_+ x_j
    \end{align*}

    \item \textbf{Elastic net} For a elastic net $r(x) = \lambda_1 \|x\|_1 + \lambda_2 \|x\|_2^2$,
    \begin{align*}
    	\left( \prox_{g,U}(x) \right)_i = \text{sign}(x_i)\left(\frac{1}{\lambda_2\eta + 1}|x_i| - \frac{\lambda_1\eta}{\lambda_2 \eta + 1} )\right)_+ 
    \end{align*}
    
    \item \textbf{Nonnegative constraint.} Let $r(x) = \mathbf{1}(x\geq 0)$ be the nonnegative constraint. Then  
    \begin{align*}
    	\prox_{r,\lambda\eta}(x) = (x)_+ 
    \end{align*}
    
    \item \textbf{Simplex constraint} Let $r(x) = \mathbf{1}(x\geq 0, \mathbf{1}^Tx=1)$ be the simplex constraint. Then 
    for $U = \text{Diag}(u)$, 
    \begin{align*}
    	\left( \prox_{r,\eta\lambda}(x)\right)_i = (x_i - \eta \lambda \nu)_+,
    \end{align*}
    Here, $\nu$ is the solution satisfying $\sum_i (x_i - \eta\lambda \nu)_+ = 1$, which can be found efficiently via bisection.

\end{itemize}


\section{Proof for Convergence Analysis of S-PI}
Let's define $\Sigma(K) = \E_{x_0\sim \mathcal{D}}[\sum_{t=0}^\infty x_tx_t^T]$. We often adopt and modify several techincal Lemmas like perturbation analysis from \cite{fazel2018global}.
\begin{lemma}[modification of Lemma 16 in \cite{fazel2018global}]
\label{lemma:perturbation} 
    Suppose $A + BK$ is stable and $K^\prime$ is in the ball $\mathcal{B}(K;\rho_K)$, i.e., 
\begin{align*}
    K^\prime \in \mathcal{B}(K;\rho_K) :=
\left\{K + \Delta K \in \reals^{m \times n} \mid 
\| \Delta K \|\leq \rho_K
\right\}
\end{align*}
where the radius $\rho_K$ is 
\[
\rho_K = \frac{\sigma_{\mathrm{min}}(\Sigma_0)  } {4 \|\Sigma(K)\| \left(\|A + BK\| + 1\right) \|B\|}.
\]
Then 
\begin{align}
    \|\Sigma(K^\prime) - \Sigma(K)\| \leq  \|\Sigma(K)\|
\end{align}
\end{lemma}

\begin{lemma}[Lemma 2 restated]\label{lemmma:smoothness}
For $K$ with stable $A + BK$, 
$f(K)$ is locally smooth with 
\begin{align*}
L_K = 4 \|\Sigma(K)\| \|R + B^T P(K) B\| < \infty,
\end{align*}
within local ball around $K \in \mathcal{B}(K;\rho_K) $

And $f(K)$ is (globally) strongly convex with 
\begin{align*}
m = \sigma_{\mathrm{min}}(\Sigma_0) \sigma_{\mathrm{min}}(R) \geq 0.
\end{align*}

In addition, $A+BK^\prime$ is stable for all $K^\prime \in \mathcal{B}(K;\rho_K)$.
\end{lemma}
\begin{proof}

First, we describe the terms with Talyor expansion 
\begin{align*}
f(K^\prime) =&f(K) - 2\Tr ( \Delta K^T\big(\!\! \left( R  + B^T P(K) B\right)K 
\\
& +  B^T P(K) A \big) \Sigma(K) ) \\
& + 
\underbrace{
\Tr\left( \Sigma(K^\prime)\Delta K^T (R + B^T P(K) B)\Delta K  \right)
}
_{\textcircled{1}}
\end{align*}
The second order term $\textcircled{1}$ is (locally) upper bounded by
\begin{align*}
\textcircled{1}
&\leq  
\|\Sigma(K^\prime)\| \|R + B^T P(K) B\| 
\|\Delta K\|_F^2
\\&
\overset{(a)}{\leq}  
2\|\Sigma(K)\| \|R + B^T P(K) B\| 
\|\Delta K\|_F^2
\end{align*}
where (a) holds due to Lemma~\ref{lemma:perturbation} 
\[
\|\Sigma(K^\prime)\| \leq \|\Sigma(K)\|  + \|\Sigma(K) - \Sigma(K^\prime)\| 
\leq 
2 \|\Sigma(K)\| 
\]
within a ball  $K^\prime \in  \mathcal{B}(K;\rho_K)$.

On the other hand, 
\begin{align*}
\textcircled{1}
&\geq  
\sigma_{\min}(\Sigma(K^\prime)) \sigma_{\min}(R + B^T P(K) B))
\|\Delta K\|_F^2
\\&
\overset{(b)}{\geq}  
\sigma_{\min}(\Sigma_0) \sigma_{\min}(R)
\|\Delta K\|_F^2
\end{align*}
where (b) hold due to 
$\Sigma _0  \preceq \Sigma(K^\prime)$ and $R \preceq R + B^T P(K) B$.

Therefore, the second order term is (locally) bounded by 
\[
\frac{m}{2} \|\Delta K\|_F^2
\leq
\textcircled{1}
\leq  
\frac{L_K}{2} \|\Delta K\|_F^2
\]
where
\begin{align*}
m =& 2\sigma_{\min}(\Sigma_0) \sigma_{\min}(R) \geq 0,
\\
L_K=&2\|\Sigma(K)\| \|R + B^T P(K) B\| < \infty.     
\end{align*}
\[
\]

\end{proof}

\begin{lemma}[Lemma 3 restated]\label{lemma:stable_ball}
   Let $K^+ = \prox_{ r,\lambda\eta}(K - \eta \nabla f(K) )$. 
   Then 
   \[
    K^+ \in  \mathcal{B}(K ;\rho_K)
   \]
    holds for any $0 < \eta < \eta_K^{\lambda,r}$ where $\eta_K^{\lambda,r}$ is given as 
    \begin{align*}
    \eta_K^{\lambda,r}= 
    \begin{cases}
      \frac{\rho_K}{\|\nabla f(K)\| + \lambda nm}
      &~ r(K) = \|K\|_1
      \\
        \frac{\rho_K}{\|\nabla f(K)\| + \lambda \min(n,m)}
      &~  r(K) = \|K\|_*
      \\
        \frac{\rho_K}{2 \|\nabla f(K)\| + 2\lambda\norm{K - K^\mathrm{ref}}  }
      &~ r(K) = \|K- K^\mathrm{ref}\|_F^2
    \end{cases}.
    \end{align*}
       
       
\end{lemma}

\begin{proof}
For lasso, let $S_{\lambda\eta}$ be a soft-thresholding operator. 
\begin{align*}
    \|K^+ &- K\| \leq  \|(K - \eta \nabla f(K)) - K\| +  \|K^+ - (K - \eta \nabla f(K))\|
    \\
    &\leq 
    \eta \|\nabla f(K)\| 
    +
   \|(S_{\eta \lambda}(K - \eta \nabla f(K)) - (K - \eta \nabla f(K))\|
    \\&\leq 
    \eta \|\nabla f(K)\| 
    +
    \eta \lambda nm
    \\
    &\leq 
    \eta (\nabla f(K) + \lambda nm)
    \\
    &\leq 
    \rho_K
\end{align*}
where the last inequality holds iff 
\[
\eta 
\leq 
\frac{ \rho_K } {\|\nabla f(K)\| + \lambda nm}.
\]
For nuclear norm,
\begin{align*}
    \|K^+ &- (K - \eta \nabla f(K))\|
    \\
    &\leq 
   \|(\mathrm{TrunSVD}_{\eta \lambda}(K - \eta \nabla f(K)) - (K - \eta \nabla f(K))\|
   \\
    &\leq 
   \|U (
   \diag(S_{\lambda \eta}[\sigma_1,\ldots,\sigma_{\min(n,m)}])
    \\
    & \qquad -
   \diag(\sigma_1,\ldots,\sigma_{\min(n,m)})
   )V^T\|
    \\&\leq 
    \eta \lambda \min (n,m)
\end{align*}
Therefore, 
\begin{align*}
    \|K^+ &- K\| \leq  \|(K - \eta \nabla f(K)) - K\| +  \|K^+ - (K - \eta \nabla f(K))\|
    \\
    &\leq 
    \eta (\nabla f(K) + \lambda \min(n,m) )
    \\
    &\leq 
    \rho_K
\end{align*}
where the last inequality holds iff 
\[
\eta 
\leq 
\frac{ \rho_K } {\|\nabla f(K)\| + \lambda \min(n, m)}.
\]

For the third regulazer, 
\begin{align*}
    \|K^+ &- (K - \eta \nabla f(K))\|
    \\
    &\overset{(a)}{\leq}
    \norm{
    \frac{2\eta\lambda K^\mathrm{ref} + K - \eta \nabla f(K)}{2\eta\lambda+1}
    - (K - \eta \nabla f(K))
    }_F
   \\
    &=
    \norm{
    \frac{2\eta\lambda}{2\eta\lambda+1}(K^\mathrm{ref}-K)
    - 
    \frac{2\eta\lambda}{2\eta\lambda+1}\eta \nabla f(K)
    }_F
    \\
    &\overset{(b)}{\leq}
    2\eta\lambda
    \norm{
    K^\mathrm{ref}-K
    }_F
    +
    \eta
    \norm{
    \nabla f(K)
    }_F,
\end{align*}
where (a) holds from the closed solution of proximal operator in Lemma 1(in main paper) and (b) holds due to $\frac{2\eta\lambda}{2\eta\lambda+1} \leq 2\eta\lambda$ and $\frac{2\eta\lambda}{2\eta\lambda+1} \leq 1$. 
Therefore, using this inequality gives
\begin{align*}
    \|&K^+ - K\|
    \\
    &\leq  \|(K - \eta \nabla f(K)) - K\| +  \|K^+ - (K - \eta \nabla f(K))\|
    \\
    &\leq 
    2\eta (\norm{\nabla f(K)} + \lambda \norm{
    K^\mathrm{ref}-K
    }_F )
    \\
    &\leq 
    \rho_K,
\end{align*}
where the last inequality holds iff 
\[
\eta 
\leq 
\frac{ \rho_K } {2 (\norm{\nabla f(K)} + \lambda \norm{
    K^\mathrm{ref}-K
    }_F )}.
\]
\end{proof}

\begin{lemma}\label{lemma:key_inequality}
For any $0<\eta \leq \min(\frac{1}{L_K}, \eta_K^{\lambda, r})$, let  $K^+ = \prox_{r,\lambda\eta}(K - \eta \nabla f(K) ) = K - \eta G_\eta(K)$ where $G_\eta(K) = \frac{1}{\eta}(K - \prox_{r,\lambda\eta}(K - \eta \nabla f(K) ) )$. Then, for any $Z \in \reals^{m\times n}$,
\begin{align}
F(K^+) 
&\leq  F(Z) + G_\eta(K)^T(K-Z) -\frac{m}{2}\norm{K-Z}_F^2 
\nonumber \\
& \qquad - \frac{\eta}{2} \norm{ G_\eta(K)}^2_F \label{eq:key_inequality}
\end{align}
holds. 
\end{lemma}
\begin{proof}
For $K^+  = K - \eta G_\eta(K) $ with any $0<\eta$ and any $Z\in \reals^{m\times n}$,
we have
\begin{align*}
r(K &-  \eta G_\eta(K)  ) \\
&\overset{(a)}{\leq} 
r(Z) -\Tr \left( \partial r(K - \eta G_\eta(K)  )^T(Z - K + \eta G_\eta(K) )\right)\\
&\overset{(b)}{=}  
r(Z) -\Tr \left((G_\eta(K) - \nabla f(K))^T(Z - K + \eta G_\eta(K) )\right)\\
&\overset{}{=}  
r(Z) + \Tr \left(G_\eta(K)^T(K-Z)\right) - \eta\norm{G_\eta(K)}^2_F \\
&\qquad + \Tr \left(\nabla f(K)^T(Z - K + \eta G_\eta(K)) \right)
\end{align*}
where (a) holds due to convexity of $g$, (b) holds due to the property of subgradient on proximal operator. 
Next, for any $0 < \eta \leq   \eta_K^{\lambda, r}$, $K^+ \in \mathcal{B}(K;\rho_K)$ holds from Lemma~\ref{lemma:stable_ball} and thus $f(K)$ is locally smooth. Therefore
\begin{align}
f(K &- \eta G_\eta(K) ) 
\nonumber
\\
&\overset{(c)}{\leq } 
f(K) 
- 
\Tr \left(\nabla f(K)^T \eta G_\eta(K) \right)
+ 
\frac{L_K\eta^2}{2} \norm{ G_\eta(K)}_F^2 
\nonumber
\\ 
&\overset{(d)}{\leq } 
f(K)  
- 
\Tr \left(\nabla f(K)^T \eta G_\eta(K)\right)
+ 
\frac{\eta}{2} \norm{ G_\eta(K)}_F^2 
\nonumber
\\ 
&\overset{(e)}{\leq } f(Z) 
-
\Tr \left(\nabla f(K)^T(Z-K )\right)
-
\frac{m}{2}\norm{Z-K}_F^2
\nonumber
\\
&\quad  
- 
\Tr \left(\nabla f(K)^T \eta G_\eta(K)\right) + \frac{\eta}{2} \norm{ G_\eta(K)}_F^2 
\nonumber
\\ 
&\overset{}{=} 
f(Z) 
-  
\Tr \left(\nabla f(K)^T(Z - K + \eta G_\eta(K)) \right)
\nonumber
\\
&\qquad -\frac{m}{2}\norm{Z-K}_F^2 + \frac{\eta}{2} \norm{ G_\eta(K)}_F^2  
\label{eq:linesearch_criterion}
\end{align}
where (c) holds due to $L$-smoothness for $K^+ \in \mathcal{B}(K;\rho_K)$, (d) holds by $\eta \leq \frac{1}{L_K}$, (e) holds due to $m$-strongly convexity at $K$. And note that Substituting $Z=K$ in \eqref{eq:linesearch_criterion} is equivalent to linesearch criterion in Eq. (8) (in main paper), which will be satisfied for small enough stepsize $\eta$ after linesearch iterations.

Adding two inequalities above gives
\begin{flalign}
F(K^+)  = &f(K - \eta G_\eta(K) ) + r(K - \eta G_\eta(K) ) \nonumber &&\\
&\leq  F(Z) + \Tr \left(G_\eta(K)^T(K-Z)\right)
\\
&\quad -\frac{m}{2}\norm{Z-K}_F^2 - \frac{\eta}{2} \norm{ G_\eta(K)}^2_F \; \nonumber && 
\end{flalign}  
\end{proof}

\begin{proposition}[Proposition 1 restated]
 \label{lemma:descent_lemma}
Assume $A + BK$ is stable. For any stepsize $0 < \eta \leq \min(\frac{1}{L_K}, \eta_K^r)$ and next iterate $K^+ = \prox_{r(\cdot),\eta \lambda} (K - \eta \nabla f(K))$, 
\begin{align}
    \rho(A + BK^+) < 1 \label{eq:stable_next_iterate}
    \\
    F(K^+) \leq  F(K) - \frac{1}{2\eta}\|K - K^+\|^2_F\label{eq:descent}
\end{align}
holds.
\end{proposition}

\begin{proof}
From Lemma~\ref{lemma:stable_ball}, \eqref{eq:stable_next_iterate} comes immediately from Lemma~8 in \cite{fazel2018global}. 
And applying $Z=K$ in Lemma~\ref{lemma:key_inequality} gives \eqref{eq:descent}. 
\end{proof}

\begin{lemma}[Lemma 4 restated]\label{lemma:bound_p_sigma}
Assume that $\{K^i\}_{i=0,\ldots}$ is a stabilizing sequence and associated $\{f(K^i)\}_{i=0,\ldots}$ and $\{\|K^i - K^\star \|_F\}_{i=0,\ldots}$ are decreasing sequences. 
Then, Lemma~\ref{lemma:stable_ball} holds for 
\begin{equation}
\label{eq:eta_r}
    \begin{aligned}
    \eta^{\lambda,r}= 
    \begin{cases}
      \frac{\rho^L}{\rho^f + \lambda nm}
      &~r(K) = \|K\|_1
      \\
        \frac{\rho^L}{\rho^f + \lambda \min(n,m)}
      &~  r(K) = \|K\|_*
      \\
        \frac{\rho^L}{2 \rho^f + 4\lambda\Delta  }
      &~ r(K) = \|K- K^\mathrm{ref}\|_F^2
    \end{cases}.
    \end{aligned}
\end{equation}
where 
\begin{align*}
\rho^f&=
2\frac{F(K^0)}{\sigma_{\min}(Q)} \bigg(
\norm{B^T}  \frac{F(K^0)}{\sigma_{\min}(\Sigma_0)} \norm{A} 
+\nonumber
\\  
&\left( \norm{R}  + \norm{B^T} \frac{F(K^0)}{\sigma(\Sigma_0)} \norm{B}\right)
(\Delta + \norm{K^\star})\bigg),
\\
\rho^L &= \frac{\sigma_{\mathrm{min}}(\Sigma_0)^2  } {8  F(K^0) \|B\| }.
\end{align*}

\end{lemma}

\begin{proof}
For the proof, we derive the global bound on $\|\nabla f(K^i)\|\leq \rho^f$ and $\rho_K\geq \rho^L$, then plug these into Lemma~\ref{lemma:stable_ball} to complete our claim.
First, we utilize the  derivation of the upperbound on $\|P(K^i)\|$ and $\|\Sigma(K^i)\|$ in \cite{fazel2018global} under the assumption of decreasing sequence as follows,
\[
    \|P(K^i)\| \leq \frac{F(K^0)}{\sigma_{\min}(\Sigma_0)}
    ,\qquad
    \|\Sigma(K^i)\|\leq \frac{F(K^0)}{\sigma_{\min}(Q)}.
    \]
From this, we have
\[
\rho_K = \frac{\sigma_{\mathrm{min}}(\Sigma_0)  } {4 \|\Sigma(K)\| \left(\|A + BK\| + 1\right) \|B\|}
\geq 
\frac{\sigma^2_{\mathrm{min}}(\Sigma_0)  } {8 F(K^0)  \|B\|}
\]
holds where we used the fact that $\|A + BK^i\|<1$ and $\|\Sigma(K^i)\|\leq \frac{F(K^0)}{\sigma_{\min}(Q)}$.
Now we complete the proof by also providing $\rho^f$
.Since $\|\nabla f(K^i)\|$ is bounded as 
\begin{align*}
\|\nabla &f(K^i)\| \leq 
    \left\|2\left(\left( R  + B^T P B\right)K +  B^T P A \right) \Sigma\right\|
    \\
    &\leq 
    2\big(\left( \norm{R}  + \norm{B^T} \norm{P} \norm{B}\right)\norm{K} 
    \\ &\quad +  \norm{B^T} \norm{P} \norm{A} \big) \norm{\Sigma}
    \\ 
    &\leq 
    2\bigg(\left( \norm{R}  + \norm{B^T} \frac{F(K^0)}{\sigma_\mathrm{min}(\Sigma_0)} \norm{B}\right)\norm{K}
    \\
    & \quad +  \norm{B^T}  \frac{F(K^0)}{\sigma_\mathrm{min}(\Sigma_0)} \norm{A} \bigg)  \frac{F(K^0)}{\sigma(Q)}
    \\ 
    &\leq 
    2\bigg(\left( \norm{R}  + \norm{B^T} \frac{F(K^0)}{\sigma_\mathrm{min}(\Sigma_0)} \norm{B}\right)
    \left(\norm{K^\star} +\Delta \right)
    \\
    & \quad +  \norm{B^T}  \frac{F(K^0)}{\sigma(\Sigma_0)} \norm{A} \bigg)  \frac{F(K^0)}{\sigma_\mathrm{min}(Q)}
\end{align*}
where the last inequality holds due to $\norm{K}\leq \left(\norm{K^\star} + \norm{K- K^\star} \right)\leq \norm{K^\star} +\Delta$.

\end{proof}




\begin{proposition} \label{proposition:startionary}
Let $\eta_i$ be the stepsize from backtracking linesearch at $i$-th iteration. After $N$ iterations, it converges to a stationary point $K^\star$ satisfying
	\begin{align*}
		\min_{i=1,\ldots, N}\|G_{\eta_i}(K^i)\|^2_F \leq \frac{2(F(K^0) - F^\star)}{\eta_{\min} N}
	\end{align*}
    where $G_{\eta_i}(K^i) \in \nabla f(K^i) + \partial r(K^i - \eta_i \nabla f(K^i))$, 
    $G_{\eta_i}(K^i)=0$ iff $0 \in \partial F(K^i)$. Moreover,

\begin{align*}
    \eta_{\min} =
    h_{\eta}\bigg(&\sigma_\mathrm{min}(\Sigma_0),
    \sigma_\mathrm{min}(Q),\frac{1}{\lambda},
    \nonumber
    \\&
    \frac{1}{\norm{A}},
    \frac{1}{\norm{B}},
    \frac{1}{\norm{R}},
    \frac{1}{\Delta},
    \frac{1}{F(K^0)}
    \bigg)
    \end{align*}

    where $h_\eta$ is a function non-decreasing on each argument. 
    
\end{proposition}

\begin{proof}
From Lemma \ref{lemma:descent_lemma}, 
\begin{align*}
F(K^{i+1}) 
\leq  
F(K^i)  - \frac{{\eta_i}}{2} \| G_{\eta_i}(K^i)\|_F^2 \; \forall i
\end{align*}
Reordering terms and averaging over iterations $i = 1 \ldots N$ give
\begin{align*}
\frac{1}{N} \sum_{i=1}^N  \| G_{\eta_i}(K^i)\|_2^2 
&\leq \frac{2}{N} \sum_{i=1}^N \frac{1}{\eta_i}( F(K^i)  - F(K^{i+1}) )
\\
&\leq \frac{2(F(K^0)  - F(K^\star)) }{\underset{i=1,\ldots,N}{\min}\eta_i N} .
\end{align*}
And LHS is lower bounded by 
\[
\frac{1}{N} \sum_{i=1}^N  \| G_{\eta_i}(K^i)\|_F^2 
\geq
\underset{{i=1,\ldots,N}}{\min }  \| G_{\eta_i}(K^i)\|_F^2,
\]
giving the desirable result. Moreover, it converges to the stationary point since $\lim_{i\rightarrow \infty} G_{\eta_i}(K^i) = 0$. 

Now the remaining part is to bound the stepsize. Note that the stepsize $\eta_i$ after linesearch satisfies
\[
\eta_i \geq  \frac{1}{\beta}\min\left(\frac{1}{L_{K_i}}, \eta_{K_i}^r\right).
\]

First we bound $\frac{1}{L_{K_i}}$ as follows,
\begin{align*}
\frac{1}{  L_{K^i}} &= \frac{1}{4\|\Sigma(K)\| \|R + B^T P(K) B\|}
\nonumber
\\ &\geq 
\frac{1}{4  \|\Sigma(K)\| (\|R \| + \|B^T\| \|P(K)\| \|B\|)}
\nonumber
\\ &\geq 
\frac{\sigma_{\min}(\Sigma_0) \sigma_{\min}(Q) }{4  F(K^0) (\sigma_{\min}(Q)\|R \| +  F(K^0) \|B^T\| \|B\|)}
    \label{eq:bound_L_K}.
\end{align*}

Next, about the bound on $\eta_{K_i}^{\lambda, r}$, we already have $\eta_{K_i}^{\lambda, r}\geq \eta^{\lambda, r}$ from Lemma~\ref{lemma:bound_p_sigma}.

Note that both of bounds are proportional to 
$\sigma_\mathrm{min}(\Sigma_0)$ and $
\sigma_\mathrm{min}(Q)
$,
    and inverse-proportional to 
$\norm{A},\norm{B},\norm{R},\Delta$ and $F(K^0)$.

Therefore

\begin{align*}\min_{i=1,\ldots,}\eta_i 
\geq 
    \eta_{\min} =
    h_{\eta}\bigg(&\sigma_\mathrm{min}(\Sigma_0),
    \sigma_\mathrm{min}(Q),\frac{1}{\lambda},
    \nonumber
    \\&
    \frac{1}{\norm{A}},
    \frac{1}{\norm{B}},
    \frac{1}{\norm{R}},
    \frac{1}{\Delta},
    \frac{1}{F(K^0)}
    \bigg)
    \end{align*}
for some $h_\eta$ that is non-decreasing on each argument.

\end{proof}

\begin{theorem}[Theorem 1 restated]\label{thm:linear_convergence}
$K^i$ from Algorithm 1 converges to the stationary point $K^\star$. Moreover, it converges linearly, i.e., after $N$ iterations,
\[
\norm{K^N - K^\star }_F^2 
\overset{}{\leq } \left(1 - \frac{1}{\kappa}\right)^N \norm{K^0 - K^\star}_F^2.
\]
Here, $\kappa=1/\left(\eta_{\min}\sigma_{\mathrm{min}}(\Sigma_0) \sigma_{\mathrm{min}}(R) ) \right)> 1$ where  
    \begin{align}
    \eta_{\min} =
    h_{\eta}\bigg(&\sigma_\mathrm{min}(\Sigma_0),
    \sigma_\mathrm{min}(Q),\frac{1}{\lambda},
    \nonumber
    \\&
    \frac{1}{\norm{A}},
    \frac{1}{\norm{B}},
    \frac{1}{\norm{R}},
    \frac{1}{\Delta},
    \frac{1}{F(K^0)}
    \bigg),
    \label{eq:stepsize}
    \end{align}
    for some  non-decreasing function $h_\eta$ on each argument.
\end{theorem}

\begin{proof}
Substituting $Z = K^\star$ in Lemma \ref{lemma:key_inequality} gives,
\begin{align*}
&F(K^+) - F^\star  
\\
&\leq 
\Tr \left(G_\eta(K)^T(K-K^\star)\right) 
-
\frac{m}{2}\norm{K-K^\star}_F^2 - \frac{\eta}{2} \norm{ G_\eta(K)}_F^2 
\\
& = \frac{1}{2\eta} \left(\norm{K-K^\star}_F^2  - \norm{K-K^\star - \eta G_\eta(K) }_F^2  \right)
\\& \qquad -\frac{m}{2}\norm{K-K^\star}_F^2\\
& = \frac{1}{2\eta} \left(\norm{K-K^\star}_F^2  - \norm{K^+-K^\star }_F^2  \right)
-\frac{m}{2}\norm{K-K^\star}_F^2.
\end{align*}
Reordering terms gives
\begin{align*}
\norm{K^+ - K^\star }^2_F 
&\leq \norm{K - K^\star}^2_F 
 \\
&\quad 
- \big( 2\eta (F(K^+) - K^\star)+ m\eta \norm{K-K^\star}^2_F \big) 
\\
&\overset{}{\leq } (1 - m\eta)\norm{K - K^\star}^2 _F 
\end{align*}
where the last inequality holds due to $F(K^+) - F^\star \geq 0$.

Therefore, after $N$ iterations,
\begin{align*}
\norm{K^N - K^\star }^2_F 
&\overset{}{\leq } (1 - m\eta_N)\cdots (1 - m\eta_1)\norm{K^0 - K^\star}^2_F  
\\
&\overset{}{\leq } (1 - m\eta_{\min})^N\norm{K^0 - K^\star}^2 
\end{align*}
where $\eta_{\min}$ is the same one in Proposition~\ref{proposition:startionary}
\end{proof}

\begin{corollary}\label{corollary:linear_convergence}
Let $K^\star$ be the stationary point from Algorithm 1.  Then, after $N$ iterations
\[
N \geq 2\kappa \log\left(\frac{\norm{K^0 - K^\star}_F}{\epsilon}\right),
\]
\[
\norm{K^N - K^\star }_F \leq \epsilon
\]
holds 
where $\kappa=1/\left(\eta_{\min}\sigma_{\mathrm{min}}(\Sigma_0) \sigma_{\mathrm{min}}(R) ) \right)> 1$ and  $\eta_\mathrm{min}$ in Eq.~\eqref{eq:stepsize}. 
\end{corollary}
\begin{proof}
This is immediate from Theorem~\ref{thm:linear_convergence}, using the inequality $(1-1/\kappa)^N\leq e^{-N/\kappa}$ and by taking the logarithm. 
\end{proof}

\section{Proof for Convergence Analysis of Model-free S-PI}
\begin{lemma}[Lemma 30 in \cite{fazel2018global}]\label{lemma:gradient_estimate}
    There exists polynomials , $h_{N_\mathrm{traj}}$, $h_{H}$, $h_r$ such that, when $r < 1/h_r(1/\epsilon)$, $N_\mathrm{traj}\geq h_{N_\mathrm{traj}}(n, 1/\epsilon, \frac{L^2}{\sigma_{\mathrm{min}}(\Sigma_0)})$ and $H\geq h_{H}(n, 1/\epsilon)$, the gradient estimate $\widehat{\nabla f(K)}$ given in Eq. (13) of Algorithm 3 satisfies
\[
\norm{\widehat{\nabla f(K)} - {\nabla f(K)}}_F \leq  \epsilon
\]
with high probability (at least $1-{(\epsilon/n)}^{n}$.
\end{lemma}

\begin{theorem}[Theorem 2 restated]\label{thm:linear_convergence_model_free}
Suppose $f(K^0)$ is finite, $\Sigma_0 \succ 0$, and that $x_0 \sim \mathcal{D}$ has norm bounded by $L$ almost surly. Suppose the parameters in Algorithm 3 are chosen from 
\[
(N_\mathrm{traj}, H, 1/r) = h\left(n, \frac{1}{\left(\sigma_{\mathrm{min}}(\Sigma_0) \sigma_{\mathrm{min}}(R)\right)}, \frac{L^2}{\sigma_{\mathrm{min}}(\Sigma_0)}\right).
\]
for some polynomials $h$.
Then, with the same stepsize in Eq.~\eqref{eq:stepsize}, Algorithm 3 converges to its stationary point $K^\star$ with high probability. In particular, there exist iteration $N$  at most $4\kappa \log\left(\frac{\norm{K^0 - K^\star}_F}{\epsilon}\right)$ such that $\norm{K^N - K^\star }\leq \epsilon$ with  at least $1 - {o}(\epsilon^{n-1})$ probability. Moreover, it converges linearly,
\[
\norm{K^i - K^\star }^2 
\overset{}{\leq } \left(1 - \frac{1}{2\kappa}\right)^i \norm{K^0 - K^\star}^2, 
\]
for the iteration $i=1,\ldots,N$,  where $\kappa=\eta\sigma_{\mathrm{min}}(\Sigma_0) \sigma_{\mathrm{min}}(R)  > 1$.
\end{theorem}

\begin{proof}

Let $\epsilon$ be the error bound we want to obtain, i.e., $\norm{K^N - K^\star }\leq \epsilon$ where $K^N$ is the policy from Algorithm 3 after $N$ iterations.  

For a notational simplicity, we denote $K \leftarrow K^i$ and see the contraction of the proximal operator at $i$th iteration. 
First we use Lemma~\ref{lemma:gradient_estimate} to claim that, with high probability, $\norm{\widehat{\nabla f(K)} - {\nabla f(K)}}_F \leq  \alpha\epsilon$ for long enough numbers of trajactory $N_\mathrm{traj}$ and horizon $H$ where $\alpha$ is specified later. 

Second, we bound the error after one iteration of approximated proximal gradient step at the policy $K$, i.e., $\norm{K^\prime - K^+ }_F$. Here let $K^\prime = \prox_{}(K - \eta \widehat{\nabla f(K)}$ be the next iterate using approximate gradient $\widehat{\nabla f(K)})$ and $K^{+} = \prox_{\lambda r}(K - \eta \nabla f(K)$ be the one using the exact gradient $\nabla f(K)$. 
\begin{align*}
    \norm{K^\prime - K^+ }_F
    &= 
    \norm{\prox_{}(K - \eta \widehat{\nabla f(K)} - \prox_{}(K - \eta \nabla f(K) }_F
    \\
    &\leq
    \norm{(K - \eta \widehat{\nabla f(K)} - (K - \eta \widehat{\nabla f(K)}}_F
    \\
    &=
    \eta \norm{\widehat{\nabla f(K)} -\nabla f(K)}_F
    \\
    &\leq 
    \eta \alpha \epsilon
\end{align*}
where we use the fact that proximal operator is non-expansive and $\norm{\widehat{\nabla f(K)} - {\nabla f(K)}}_F \leq  \alpha\epsilon$ holds for proper parameter choices (the claim in the previous paragraph).

Third, we find the contractive upperbound after one iteration using approximated proximal gradient. 
\begin{align*}
\norm{K^\prime - K^\star }_F 
&\leq 
\norm{K^\prime - K^+ }_F
+
\norm{K^+ - K^\star}_F 
\\
&\leq 
\norm{K^\prime - K^+ }_F 
+
\sqrt{(1 - 1/\kappa)}\norm{K - K^\star}_F 
\\
&\leq 
\eta\alpha\epsilon
+
\sqrt{(1 - \kappa)}\norm{K - K^\star}_F. 
\end{align*}
Let's assume $\norm{K-K^\star}_F \geq \epsilon$ under current policy. Then,
taking square on both sides gives
\begin{align*}
\norm{K^\prime - K^\star }^2_F
&\leq 
\eta^2\alpha^2\epsilon^2
+
2 \eta\alpha\epsilon \sqrt{(1 - 1/\kappa)}\norm{K - K^\star}_F 
\\
& \quad +
(1 - 1/\kappa)\norm{K - K^\star}^2_F 
\\
&\leq
\left(1 -1/\kappa + 2\alpha \eta  + \alpha^2 \eta^2 \right)\norm{K - K^\star}^2_F, 
\end{align*}
where we used $\sqrt{(1 - 1/\kappa)}\leq 1$, $1 \leq \kappa $, and the assumption.
Choosing $\alpha = \frac{1}{5\eta\kappa}= \frac{1}{5\sigma_{\mathrm{min}}(\Sigma_0) \sigma_{\mathrm{min}}(R) )}$ results in
\[
\norm{K^\prime - K^\star }^2
\leq
\left(1 -1/(2\kappa) \right)\norm{K - K^\star}^2,
\]
with high probability $1-(\alpha \epsilon/n)^n$.
This says the approximate proximal gradient is contractive, decreasing in error after one iteration. Keep applying this inequality, we get 
\[
\norm{K^{i} - K^\star }^2
\leq
\left(1 -1/(2\kappa) \right)^i\norm{K^0 - K^\star}^2. 
\]
as long as $\epsilon\leq \norm{K^{i-1}-K^\star}_F$.  

This says that there must exist the iteration $N>0$ s.t. 
\begin{align}
    \norm{K^{N}-K^\star}_F \leq \epsilon\leq \norm{K^{N-1}-K^\star}_F
    \label{eq:contradiction}
\end{align}

Now we claim this $N$ is at most $ 4\kappa \log\left(\frac{\norm{K^0 - K^\star}_F}{\epsilon}\right)$. To prove this claim, suppose it is not, i.e.,  $N \geq 4\kappa \log\left(\frac{\norm{K^0 - K^\star}_F}{\epsilon}\right) + 1$. Then, for $N>1$, 
\begin{align*}
\norm{K^{N-1}-K^\star}_F ^2
&\leq
\left(1 -1/(2\kappa) \right)^{N-1}\norm{K^0 - K^\star}^2
\\&<
e^{-\frac{N-1}{2\kappa}}\norm{K^0 - K^\star}^2
\\&\leq 
\left(\frac{\norm{K^0 - K^\star}_F}{\epsilon}\right)^{-2}
\norm{K^0 - K^\star}^2
\\&=
\epsilon^2,
\end{align*}
which is a contradiction to ~\eqref{eq:contradiction}. 

Finally, we show the probability that this event occurs. Note that  all randomness occur when estimating  $N$ gradient within $\alpha \epsilon$ error. From union bound, it occurs at least $1 - N (\alpha\epsilon/n)^{n}$. And this is bounded below by 
\begin{align*}
1 - N (\alpha\epsilon/n)^{n}
&\geq
1 - \left(4\kappa \log\left(\frac{\norm{K^0 - K^\star}_F}{\epsilon}\right) + 1\right)
(\alpha\epsilon/n)^{n}
\\&\geq
1 - {o}\left(\epsilon^n \log \left(\frac{\norm{K^0 - K^\star}_F}{\epsilon}\right)\right)
\\&=
1 - {o}(\epsilon^{n-1}).
\end{align*}

\end{proof}

\section{Additional Experiments on Stepsize Sensitivity}
In this section, we scrutinize the convergence behaviors of S-PI under some fixed stepsize. For a very small Laplacian system $(n,m)=(3,3)$ with Lasso penalty $\lambda=3000$, we run S-PI over a wide range of stepsizes. For stepsize larger than $3.7e-4$, S-PI diverges and thus is ran under stepsizes smaller than $3.7e-4$.  Let $K^{min}$ be the policy where the objective value attains its minimum among overall iterates and $K^\star$ be the policy from S-PI with linesearch (non-fixed stepsize). Here the cardinality of the optimal policy is $3$. For a fixed stepsizes in $[3.7e-5, 3.7e-6]$, S-PI converges to the optimal. In Figure~\ref{fig:convergence_large_stepsize}, the objective value monotonically decreases and the policy converges to optimal one based on errors and cardinality. However, for smaller stepsize like   $[ 3.7e-7, 3.7e-8, 3.7e-9]$, Figure~\ref{fig:convergence_small_stepsize} shows that S-PI still converges but does not show monotonic behaviors nor converges to the optimal policy. These figures demonstrate the sensitivity of a stepsize when S-PI is used under a fixed stepsize, rather than linesearch. Like in Figure~\ref{fig:convergence_small_stepsize}, the algorithm can be unstable under fixed stepsize because the next iterate $K^+$ may not satisfy the stability condition $\rho(A + BK^+)<1$ and or are not guaranteed for a monotonic decrease. Moreover, this instability may lead to another stationary point even when the iterate falls in some stable policy region after certain iterations. This not only demonstrates the importance of lineasearch due to its sensitivity on the stepsize, but may provide the evidence for why other policy gradient type of methods for LQR did not perform well in practice.

\begin{figure}
\centering
\includegraphics[width=0.8\linewidth]{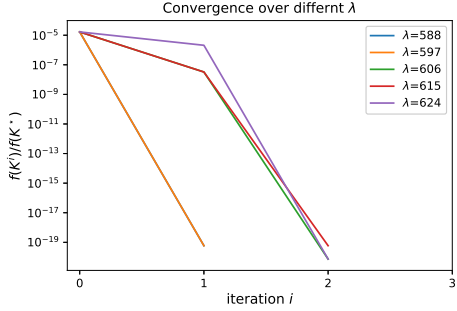}

\vspace{-5mm}
\caption{Convergence behavior of the Structured Policy Iteration (S-PI) under linesearch for Laplacian system of $(n,m)=(3,3)$ over various $\lambda$s. 
}
\label{fig:convergence_linesearch}
\end{figure}

\begin{figure}
\centering
\includegraphics[width=0.8\linewidth]{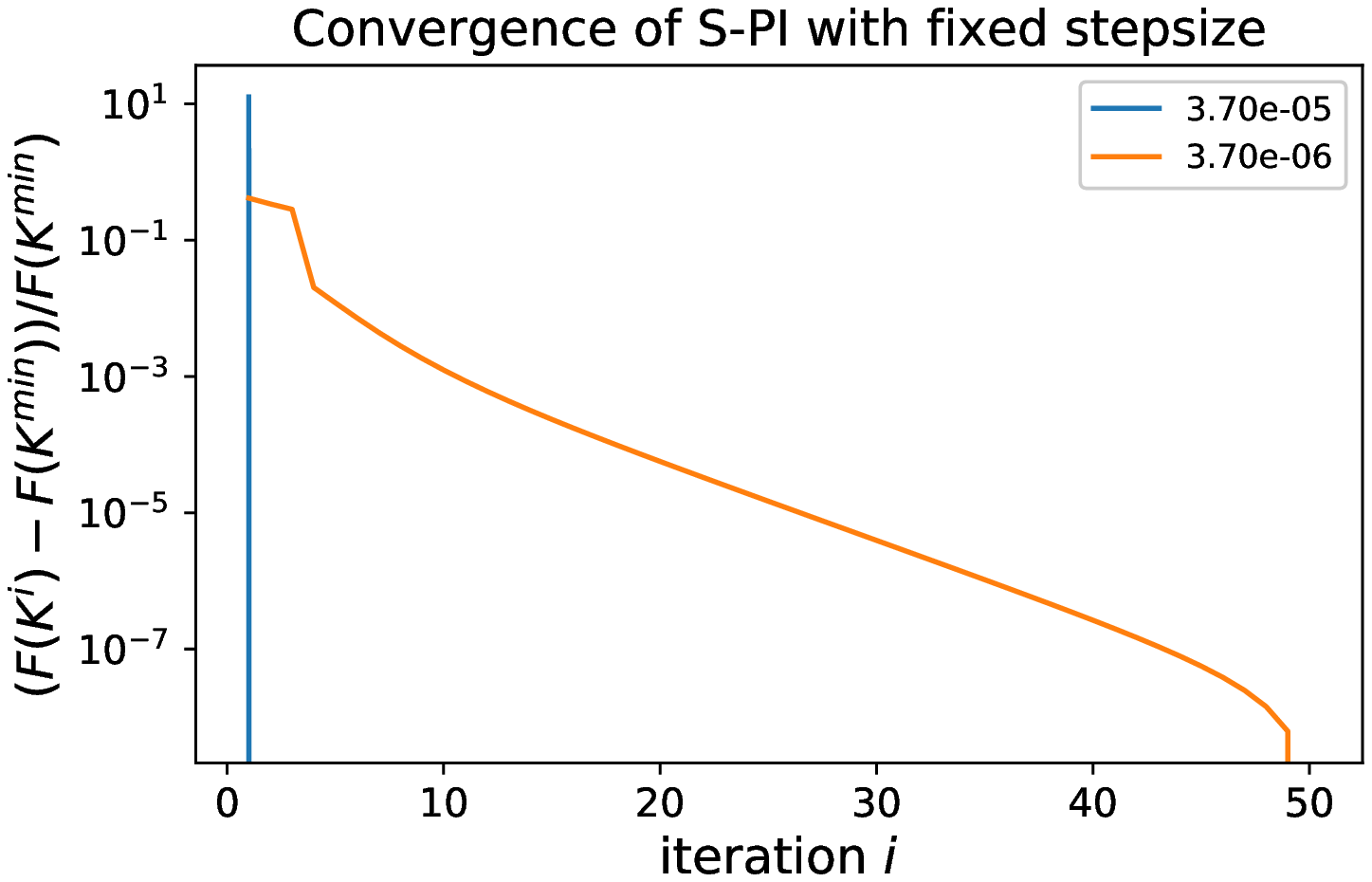}
\includegraphics[width=0.8\linewidth]{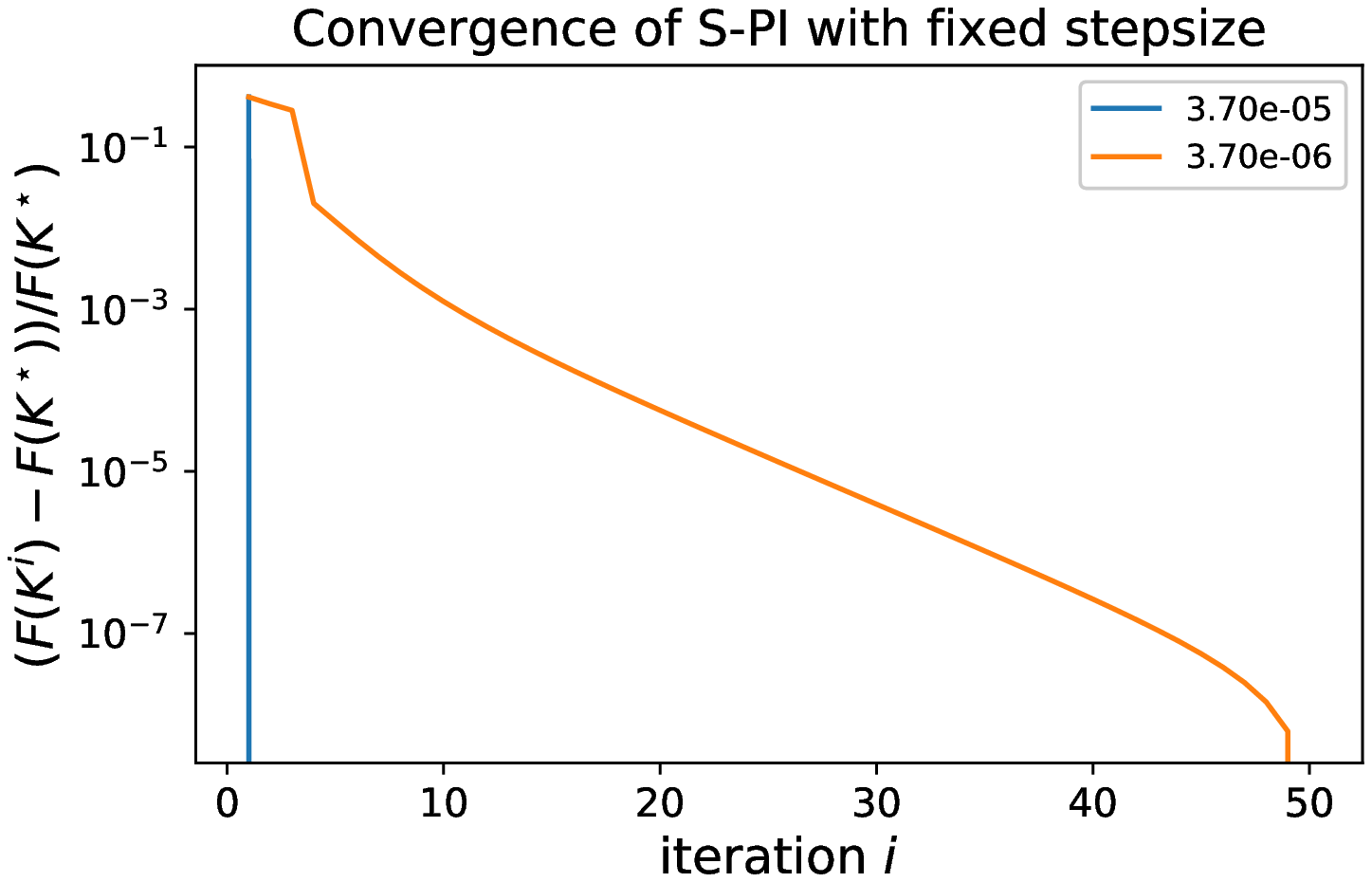}
\includegraphics[width=0.8\linewidth]{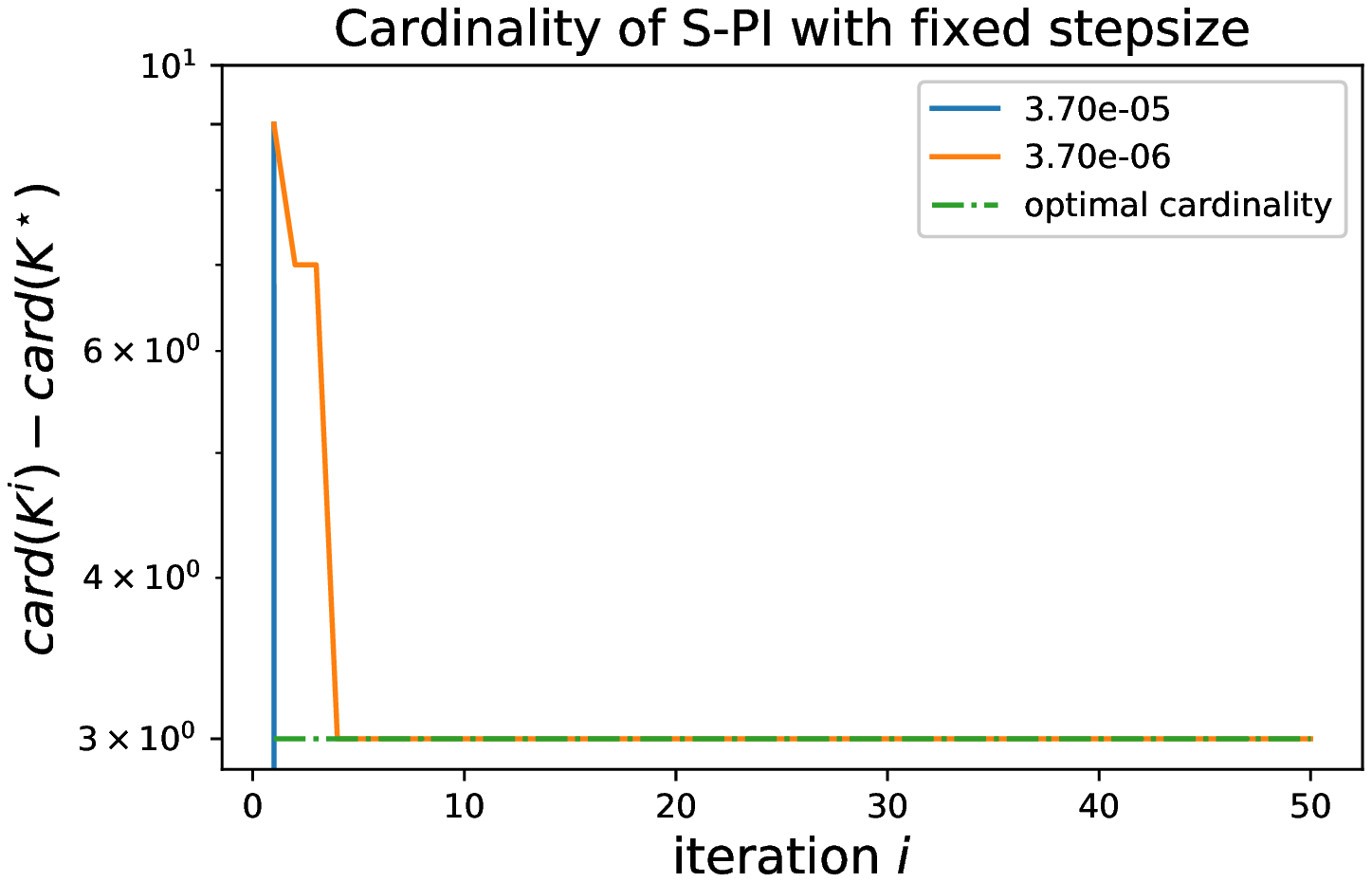}
\vspace{-5mm}
\caption{Convergence behavior of the Structured Policy Iteration (S-PI) over fixed stepsizes $[3.7e-5, 3.7e-6]$ for Laplacian system of $(n,m)=(3,3)$ with $\lambda=3000$. 
}
\label{fig:convergence_large_stepsize}
\end{figure}\textbf{}

\begin{figure}
\centering
\includegraphics[width=0.8\linewidth]{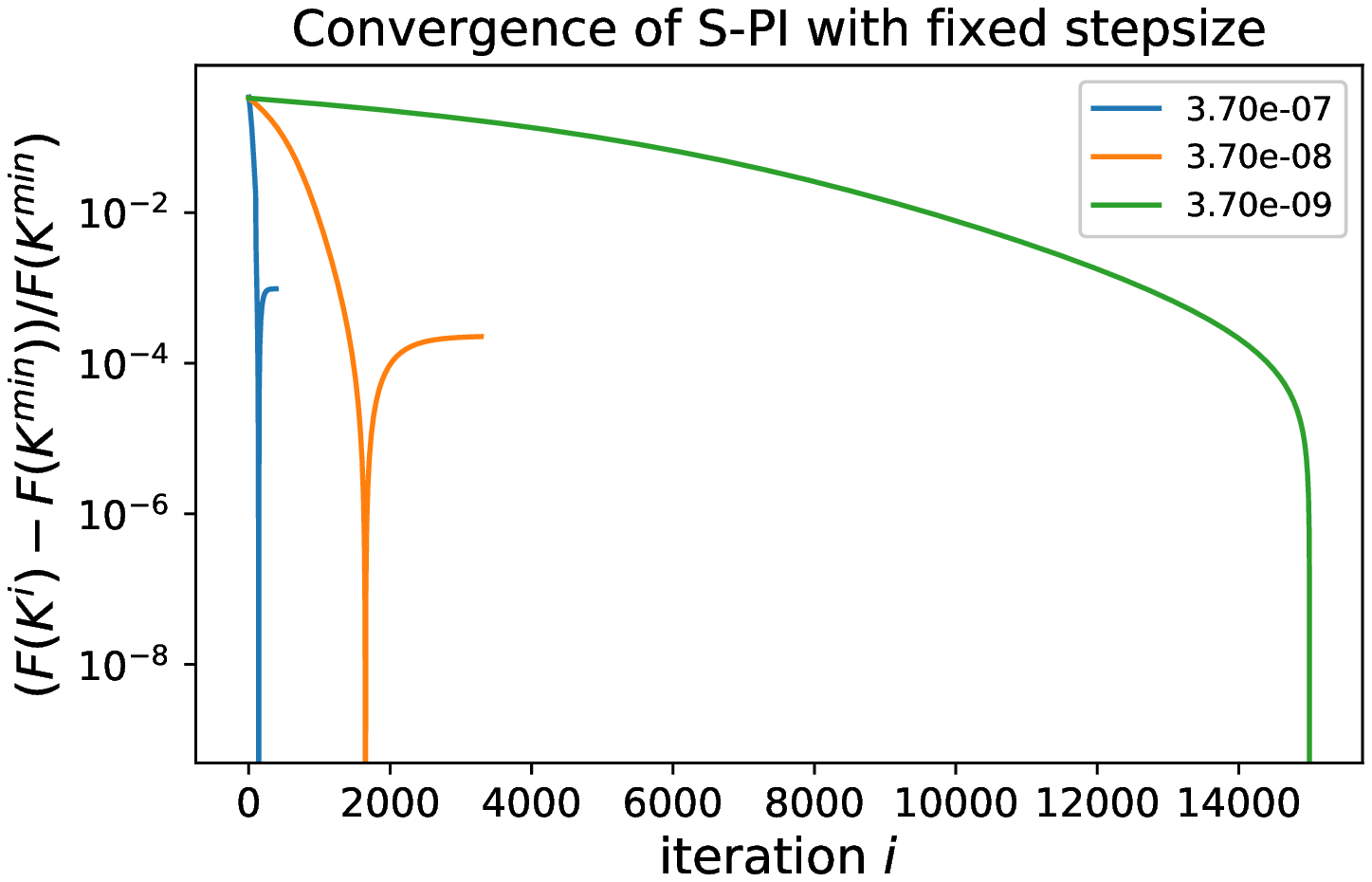}
\includegraphics[width=0.8\linewidth]{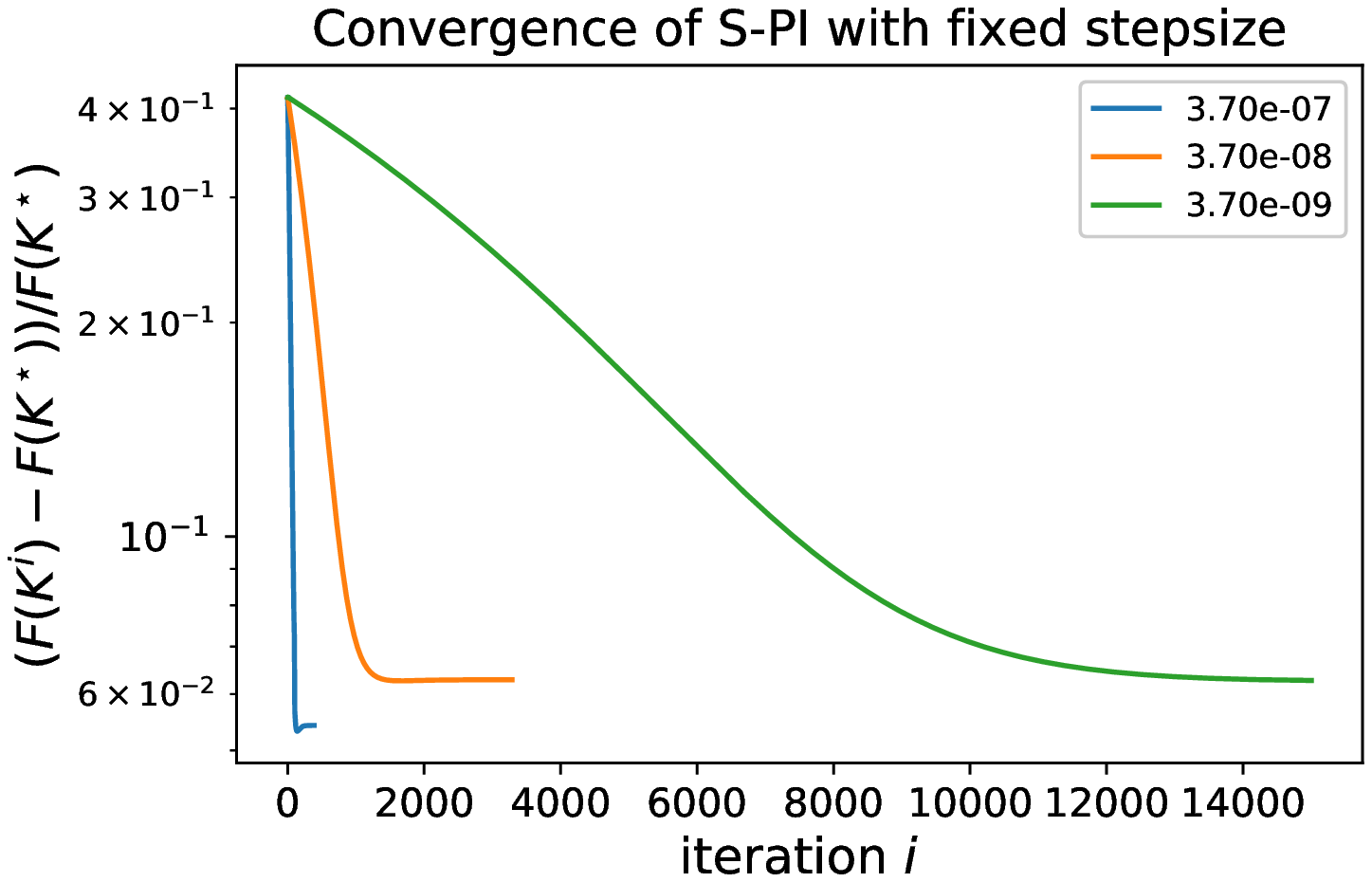}
\includegraphics[width=0.8\linewidth]{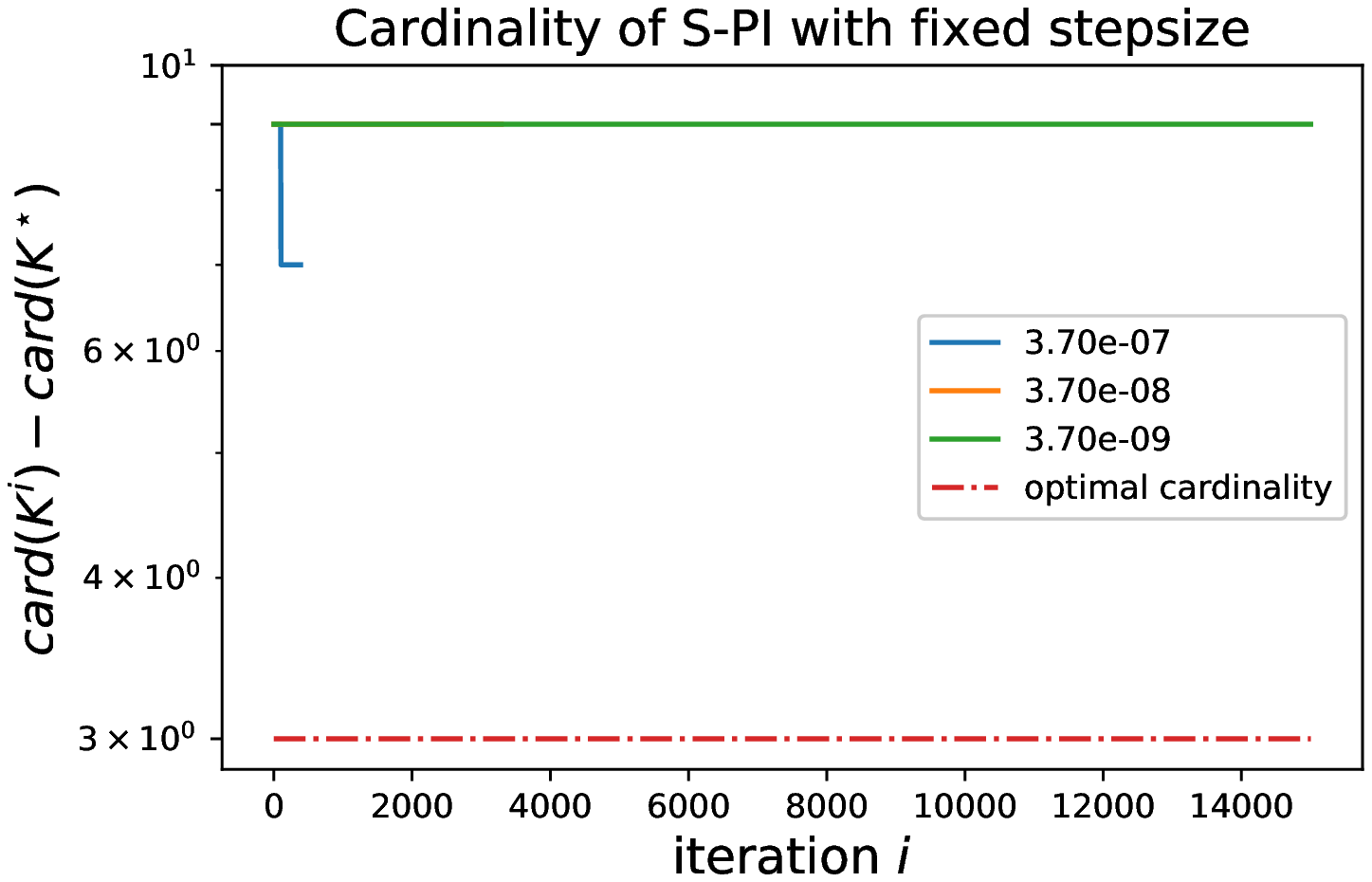}
\vspace{-5mm}
\caption{Convergence behavior of the Structured Policy Iteration (S-PI) over fixed stepsizes $[ 3.7e-7, 3.7e-8, 3.7e-9]$ for Laplacian system of $(n,m)=(3,3)$ with $\lambda=3000$. 
}
\label{fig:convergence_small_stepsize}
\end{figure}
\end{document}